%% file: main.tex
\documentclass{article}
\usepackage[margin=1in]{geometry}
\usepackage[utf8]{inputenc}
\usepackage[T1]{fontenc}
\usepackage{url}
\usepackage{booktabs}
\usepackage{amsfonts}
\usepackage{nicefrac}
\usepackage{microtype}
\usepackage{xcolor}
\usepackage{thmtools}
\usepackage{thm-restate}

\input{packages}
\input{macro}

\title{Let's Ask Gauss: Improved One-Run Privacy Auditing}
\date{}
\author{
Adya Agrawal$^{1}$ \quad
Yu Wei$^{1}$ \quad
Jaspal Singh$^{2}$ \quad
Malik Magdon-Ismail$^{3}$ \quad
Vassilis Zikas$^{1}$ \\
$^{1}$Georgia Institute of Technology \\
$^{2}$Mohamed bin Zayed University of Artificial Intelligence \\
$^{3}$Rensselaer Polytechnic Institute
}

\begin{document}

\maketitle
\input{00-abstract.tex}
\input{01-introduction.tex}

\input{02-prelim.tex}
\input{101-app-related_works.tex}
\input{03-methodology.tex}
\input{04-convergence-analysis.tex}

\input{05-CI-gaussian-estimation.tex}
\input{06-experiments.tex}

\input{07-discussion.tex}

\input{08-ack.tex}
\newpage
\bibliographystyle{plain}                                           
\bibliography{references}  
\newpage
\appendix
\input{100-app.tex}

\newpage

\end{document}

%% file: packages.tex
\usepackage{amssymb, amsfonts, amsmath, amsthm}
\usepackage{amssymb}
\usepackage[T1]{fontenc}
\usepackage{graphics, graphicx}
\graphicspath{{Figures/}}
\usepackage{amsthm}
\usepackage{subcaption}
\usepackage{longtable, multirow, multicol}
\usepackage{colortbl, color, xcolor}
\usepackage{fancyvrb}
\usepackage{enumerate, tasks}
\usepackage[colorlinks=true,linkcolor=magenta,citecolor=violet,pagebackref]{hyperref}
\usepackage{seqsplit}
\usepackage{tabularx}
\usepackage{tabu}
\usepackage{pgf}
\usepackage{pgfplots}
\pgfplotsset{compat=newest}
\usepgfplotslibrary{fillbetween}
\usepackage{todonotes}

\usepackage{orcidlink}

\usepackage{mathtools}
\usepackage{algorithm}
\usepackage[noend]{algpseudocode}
\usepackage{amsfonts}
\usepackage{braket}
\usepackage{cleveref}
\usepackage{float}
\usepackage{tikz-cd}
\usepackage{placeins}
\usepackage{caption}

\usepackage{booktabs}
\usepackage{comment}
\usepackage{enumitem}
\usepackage[title]{appendix}

\newtheorem{theorem}{Theorem}
\newtheorem{proposition}{Proposition}
\newtheorem{lemma}{Lemma}
\newtheorem{corollary}{Corollary}

\newtheorem{assumption}{Assumption}

\theoremstyle{definition}
\newtheorem{definition}{Definition}
\newtheorem{model}{Model}


%% file: macro.tex
\newcommand{\ignore}[1]{
}

\newcommand{\cancel}[1]{ }

\newcommand{\Real}{\mathbb{R}} 

\newcommand{\DB}{D} 

\newcommand{\twoNorm}[1]{{\left\Vert #1 \right\Vert}_2}

\newcommand{\clip}{\mathsf{clip}}

\newcommand{\rN}{N}
\newcommand{\rS}{S}

\newcommand{\rX}{X}
\newcommand{\rY}{Y}
\newcommand{\rZ}{Z}

\newcommand{\multiNormal}[2]{\mathcal{N}\left(#1, #2\right)}

\newcommand{\defin}{ \stackrel{\rm def}{=} }

\newcommand{\bits}{ \{0,1\} }

\newcommand{\pr}[1]{ \Pr\left[ #1 \right] }

\newcommand{\cond}{ \middle| }

\newcommand{\cM}{\mathcal{M}}

\newcommand{\cY}{\mathcal{Y}}
\newcommand{\cL}{\mathcal{L}}
\newcommand{\cS}{\mathcal{S}}

\newcommand{\Ex}[1]{\mathbf{Ex} \left[ #1 \right ]}

\newcommand{\Var}[1]{\mathbf{Var} \left[ #1 \right ]}

\newcommand{\abs}[1]{\left| {#1} \right|}

\newcommand{\Mech}{\cM} 

\newcommand{\rSum}[3]{\sum\limits_{#1 = #2}^{#3}}
\newcommand{\rProd}[3]{\prod\limits_{#1 = #2}^{#3}}
\newcommand{\rExp}[1]{\exp\left(#1\right)}

\newcommand{\ber}[1]{\mathrm{Bernoulli}(#1)}

\newcommand{\innerProd}[2]{\langle #1, #2\rangle}

\newcommand{\iid}{\mathrel{\overset{\text{i.i.d.}}{\sim}}}

\newcommand{\lb}{\mathrm{lb}}

\newcommand{\GM}[5]{\mathrm{GM}\!\left(#1,\,#2,\,#3,\,#4,\,#5\right)}
\newcommand{\dK}[2]{d_{\mathrm{K}}\!\left(#1,\,#2\right)}
\newcommand{\MGF}[1]{M_{#1}}
\newcommand{\CGF}[1]{K_{#1}}
\newcommand{\cum}[2]{\kappa_{#1}\!\left(#2\right)}
\newcommand{\auditStat}{\widetilde{\rS}_T}
\newcommand{\trueLaw}{\widetilde{Q}_T}
\newcommand{\modelLaw}{Q_T}

\providecommand{\cF}{\mathcal{F}}

%% file: 00-abstract.tex
\begin{abstract}
Privacy auditing provides an important safeguard by estimating the actual information leaked by a model, thus ensuring that theoretical privacy guarantees hold in practice. We study empirical privacy auditing for differentially private (DP) machine learning, focusing on efficient one-run methods for mechanisms such as DP-SGD. Prior one-run approaches threshold training examples or “canaries” into binary membership guesses, which discards useful information. We show that, in the white-box DP-SGD setting, canary-aligned signals naturally form a sequence of random variables whose normalized sum is asymptotically Gaussian. We show that this gaussian perspective can also be naturally extended to other composition mechanisms that satisfy a classical, triangular-array, or martingale central limit theorem. Leveraging this distributional perspective, we develop a DP-auditing framework that leads to tighter privacy lower bounds from a single training run.
\end{abstract}

%% file: 01-introduction.tex
\section{Introduction}

As data becomes increasingly valuable, a central challenge is to extract statistical utility from potentially sensitive data without compromising the privacy of the individuals who contribute to it. Differential privacy (DP)~\cite{TCC:DMNS06,dwork2014algorithmic} has emerged as a broadly accepted notion for addressing this challenge and has enabled a wide range of privacy-preserving applications, including private deep learning~\cite{AbadiCGMMT016}, synthetic data generation~\cite{dpsynthetic,USENIX:ZWLHBH21,cvprdpsynth,yoon2018pategan,pmlr-v97-mckenna19a}, the release of aggregate statistics~\cite{TCC:DMNS06,LiMHMR15,chan2011private} etc.

The growing use of DP in machine learning and data analysis has motivated a parallel need for \emph{privacy auditing}: methods that test whether a concrete implementation actually satisfies its claimed privacy specification. Such discrepancies often arise due to reasons ranging from simple coding errors to incorrect parameter settings. This need has motivated a growing body of work that ranges from using classical bug detection approaches to statistically estimating whether some
privacy-relevant event can distinguish neighboring distributions more strongly than the claimed DP upper bound (see for example~\cite{lu2024eureka,askin2025general,nasr2023tight}).

Classical auditing approaches to private deep learning mechanisms, including DP-SGD, verify the privacy of the mechanism by running it multiple times and applying a statistical test on the (aggregate) outputs~\cite{jagielski2020auditing,nasr2021adversary}. A challenge here is that DP offers worst-case guarantees, which requires finding worst case (neighboring) datasets to test the mechanism. A valuable approach suggested by Carlini et al. \cite{carlini2019secret} is to inject structured examples referred to as {\em canaries} to the training data and test the effect that these canaries have on the output. However, even this approach does not solve the problem of scalability: it requires repeating the learning process (with or without canaries) many times, making it computationally intensive, and often prohibitive. 

Recent work has addressed whether meaningful privacy lower bounds can be obtained while avoiding the overhead of repeatedly running the training algorithms. Steinke et al. \cite{steinke2023privacy} accomplished this using only one training run by embedding many independent canaries in parallel to define many neighboring dataset pairs simultaneously. The advantage of such an auditing method is evident: the cost of auditing the privacy of a model reduces to the cost of simply training the model. They demonstrate that the gradient projections of the canary examples can be thresholded into (binary) membership guesses which give a statistically valid privacy lower bound. Subsequent work has sharpened the analysis of the canary-aligned observations by using $f$-DP tradeoff curves, modeling canary inclusion bits as messages in a noisy channel, and using order statistics to select high-confidence guesses~\cite{mahloujifar2025auditing,xiang2025bits,xiang2025tight}. This one-shot paradigm has proven especially critical for federated learning, where the distributed nature of client data makes traditional, multi-run auditing impractical \cite{andrew2024oneshot}. 

Existing one-run methods for DP-SGD share a common bottleneck: each per-canary score is ultimately collapsed to a single binary membership guess before aggregation, discarding the rich distributional information of the canaries and the underlying mechanism when designing statistical tests.
While these approaches are proven to be very promising, it naturally brings forth the following questions: 
\begin{quote}
\vspace{-0.35\baselineskip}
   What are the distributional properties that underlie the general success of one-run 
   DP-SGD auditing? And,  can we characterize this distribution precisely, and exploit it to develop improved auditing? 
\vspace{-0.35\baselineskip}
\end{quote}


\cancel{Classical auditing approaches to private deep learning mechanisms
such as DP-SGD derive a statistical claim by repeatedly running the target
mechanism on neighboring datasets. The computational cost is prohibitive. Recent work has addressed whether meaningful privacy lower bounds can be obtained from one training run. The approach of \citet{steinke2023privacy} accomplishes this by 
using embedded canaries to define neighboring data sets. In one training run, gradient projections on the canary examples are thresholded into (binary) membership guesses which give a statistically valid privacy lower bound. Subsequent work has sharpened the analysis of the canary-aligned observations by using $f$-DP tradeoff curves, modeling canary inclusion bits as messages in a noisy channel, and using order statistics to select high-confidence guesses~\cite{mahloujifar2025auditing,xiang2025bits,xiang2025tight}.}

\cancel{
This approach to one-run auditing is powerful and broadly applicable. However, it throws away
distributional information when the scores obtained from the canary-aligned observations are compressed to 
binary signals for processing into a membership decision.  A natural question is:
\begin{quote}
    Can one-run white-box DP-SGD auditing be strengthened by using the distributional information in non-thresholded canary-aligned observations.
\end{quote}
}

To 
understand the distributional properties underlying these prior works, it is worth abstracting their common core: the canary-aligned observation is generated sequentially, with each training step contributing a bounded clipped signal (when the canary is sampled) plus a fresh Gaussian noise term. The final canary score accumulates these noisy observations across training, so by the Central Limit Theorem its distribution converges to a Gaussian.
This intuition is already implicit in prior auditors, which use mean-like statistics on canary-aligned observations.
We conjecture that this asymptotic Gaussian behavior is central to the accuracy of canary-based one-run auditing, and exploiting it directly rather than via a binary thresholding intermediate as in prior works, gives us tighter audits.

\noindent\textbf{Our contributions.}

\textit{A sequence-level distributional view and a Gaussian-pair auditor.} We show that the normalized canary score in one-run DP-SGD auditing is a sum of sequential random variables with a Gaussian limiting law, and we give its mean and variance in closed form as functions of the DP-SGD hyperparameters. Building on this, we design a one-run white-box auditor that models the absent- and present-canary score distributions as a pair of one-dimensional Gaussians. We use the closed-form hockey-stick divergence between Gaussians~\cite{andrew2024oneshot,lu2023normal} to obtain a much tighter lower bound on the DP distance.

\textit{Quantitative convergence guarantees.} We address whether the Gaussian asymptotics manifest within a practical number of training steps ($T$). The bound decays as $O(T^{-1/2})$ in general, and decays even faster as $O(T^{-1})$ in the small-sampling-rate regime $q = O(T^{-1/2})$ that is expected for DP-SGD. Applying this bound with our experimental parameters, we obtain a deviation three orders of magnitude smaller than the chosen value of $\delta$.

\textit{Benchmark and Comparison.} We evaluate our auditor on DP-SGD, where the sequential canary-aligned scores are asymptotically Gaussian. On CIFAR-10 DP-SGD at theoretical $\varepsilon = 8$, our auditor recovers an empirical lower bound of $\approx 6.7$ (about $84\%$ of the analytic upper bound), compared with $\approx 4.7$ for the $f$-DP audit of~\cite{mahloujifar2025auditing} and $\approx 3.3$ for~\cite{steinke2023privacy}. We obtain similar $1$--$2\times$ improvements across all $\varepsilon$ regimes. We demonstrate that our auditor detects privacy violations that previous auditors miss due to significantly looser bounds. We also extend our auditor to other composition algorithms like matrix factorization (including DP-FTRL).

%% file: 02-prelim.tex
\section{Preliminaries}

\subsection{Differential privacy and hockey-stick divergence}
Differential privacy (DP)~\cite{TCC:DMNS06} is the most widely accepted privacy definition for the release of (queries on) sensitive data. Let $\DB \sim \DB'$ denote neighboring databases, meaning that $\DB$ and $\DB'$ differ in the addition or removal of one record. 
A randomized mechanism $\Mech$ with output space $\cY$ is $(\varepsilon,\delta)$-DP if, for every neighboring pair $\DB \sim \DB'$ and every measurable event $\cS \subseteq \cY$,
\begin{align*}
    \pr{\Mech(\DB) \in \cS} \leq e^{\varepsilon}\pr{\Mech(\DB') \in \cS} + \delta .
\end{align*}
For a fixed pair of neighboring output distributions $\Mech(\DB), \Mech(\DB')$, the smallest admissible value of $\delta$ for a given $\varepsilon$ is well known as the hockey-stick divergence~\cite{TPCBalleBG20} between $\Mech(\DB)$ and $\Mech(\DB')$ at level $e^\varepsilon$. Formally, 
\begin{align*}
    \delta_{\Mech(\DB), \Mech(\DB')}(\varepsilon) = \sup_{\cS\in \Real^T} \Bigl(\pr{\Mech(\DB)\in \cS}-e^{\varepsilon}\pr{\Mech(\DB')\in \cS}\Bigr)_+,
\end{align*}
where we write $(x)_+ \defin{} \max\{x,0\}$, meaning  $\delta(\varepsilon) \geq 0$. 

In our auditor, the induced neighboring distributions are represented by a pair of one-dimensional Gaussians. For $\theta=(\mu_1,\sigma_1,\mu_2,\sigma_2)$ with $\sigma_1,\sigma_2>0$, it is therefore enough to compute the hockey-stick divergence
\begin{align}
    \label{def:delta-func-eps}
    \delta_\theta(\varepsilon) \defin \delta_{\multiNormal{\mu_1}{\sigma_1^2}, \multiNormal{\mu_2}{\sigma_2^2}}(\varepsilon) 
\end{align}
The following closed form is the one-dimensional specialization of the hockey-stick divergence calculation for arbitrary pairs of Gaussians:

\begin{lemma}[Hockey-stick divergence between one-dimensional Gaussians~\cite{andrew2024oneshot,lu2023normal}]
\label{lem:ndis-1d-a-pos}
Let $\theta=(\mu_1,\sigma_1,\mu_2,\sigma_2)$ with $\mu_1,\mu_2\in\Real$ and $0<\sigma_1^2<\sigma_2^2$, and fix $\varepsilon\geq 0$. 
Let
    \begin{align*}
        \tau &\defin \frac{\sigma_1^2}{\sigma_2^2}, & a &\defin 1-\tau, & b &\defin -\frac{\sigma_1(\mu_1-\mu_2)}{\sigma_2^2}, \\
        c &\defin \varepsilon+\frac{1}{2}\log \tau -\frac{(\mu_1-\mu_2)^2}{2\sigma_2^2}, & \Delta &\defin b^2-2ac, & m &\defin -\frac{\mu_1-\mu_2}{\sigma_1}.
    \end{align*}
If $\Delta \leq 0$, then $\delta_\theta(\varepsilon)=0$.  If $\Delta>0$, define
$  z_{\pm} \defin (-b\pm\sqrt \Delta)/a.$ 
Then, writing $\Phi$ for the standard Gaussian CDF,
\begin{align*}
    \delta_{\theta}(\varepsilon) = \Bigl(\Phi(z_+)-\Phi(z_-)\Bigr) - e^{\varepsilon} \Bigl( \Phi\bigl((z_+-m)\sqrt{\tau}\bigr) - \Phi\bigl((z_--m)\sqrt{\tau}\bigr)\Bigr).
\end{align*}
\end{lemma}
Equivalently, given a fixed $\delta$, we define the function $\varepsilon_\theta(\delta)$ as the smallest privacy parameter $\varepsilon$ for which the Gaussian pair has hockey-stick divergence at most $\delta$. Formally,
\begin{align}
    \label{def:epsilon-func-delta}
    \varepsilon_\theta(\delta) \defin
    \inf\left\{
        \varepsilon\geq 0:
        \delta_\theta(\varepsilon)\leq \delta
    \right\}.
\end{align}

\subsection{DP-SGD transcripts and canary-aligned observations}
\label{sec:prelim-dpsgd-canary}

The seminal DP-SGD algorithm~\cite{AbadiCGMMT016} is a differentially private variant of the standard stochastic-gradient method. In each iteration, the algorithm computes per-example gradients on a sampled minibatch, clips each gradient to have $\ell_2$ norm at most $C$, aggregates the clipped gradients, and adds Gaussian noise $\rZ_t \sim \multiNormal{0}{\sigma^2 C^2 I_d}$ before 
updating the model parameters $\theta_t$.

Broadly speaking, there are two general types of DP auditing for ML training mechanisms:  In {\em black-box} auditing, the auditor may observe only the final model, or possibly query access to its predictions.  In \emph{white-box} auditing, the 
auditor sees the internal state of training; we focus on the standard 
instantiation in which the auditor observes the noisy averaged gradient 
$\widetilde{g}_t^{(b)}$ at each step, consistent with the privacy 
analysis of~\cite{AbadiCGMMT016}. We use the bit $b \in \{0,1\}$ to 
indicate whether a target canary $x^\star$ is absent ($b=0$) or present 
($b=1$) in the training dataset, following~\cite{nasr2023tight}, 
and write $\mathsf{Tr}^{(b)}_T = (\widetilde{g}_t^{(b)}, \theta_t^{(b)})_{t \in [T]}$ 
for the resulting white-box transcript. The full procedure is formalized 
as~\Cref{alg:dpsgd-transcript} in Appendix \ref{app:transcript}, identical to transcript in \cite{steinke2023privacy}.

\Cref{def:canary-aligned-obs} formally defines the sequential observation used by the white-box audit. Intuitively, the auditor projects the privatized gradient at each iteration onto the direction of the canary gradient. Thus, the observation in our analysis is not the model iterate itself, but a one-dimensional canary-aligned projection of the noisy gradient update. 

\begin{definition}[Canary-aligned sequential observation](Following prior works~\cite{steinke2023privacy,mahloujifar2025auditing}
\label{def:canary-aligned-obs})
Let $b\in\bits$, and let  $\mathsf{Tr}^{(b)}_T=\bigl(\widetilde g_t^{(b)},\theta_t^{(b)}\bigr)_{t\in[T]}$ be the white-box DP-SGD transcript from~\Cref{alg:dpsgd-transcript}. Let $\bar u$ be a deterministic unit vector, and for each iteration $t\in[T]$, define
\begin{align*}
    c_t^{(b)} \defin \clip_C\!\left(\nabla_{\theta}\ell(\theta_{t-1}^{(b)};x^\star)
    \right),  \qquad
    u_t^{(b)} \defin
    \begin{cases}
        c_t^{(b)}/\twoNorm{c_t^{(b)}} & \text{if } c_t^{(b)}\neq 0,\\
        \bar u & \text{otherwise}. 
    \end{cases}
\end{align*}
The \emph{canary-aligned sequential observation} induced by $\mathsf{Tr}^{(b)}_T$ and $x^\star$ is
\begin{align*}
    X_{1:T}^{(b)} \defin \bigl(X_1^{(b)},\ldots,X_T^{(b)}\bigr) 
    \qquad \text{where,} \quad X_t^{(b)} \defin \innerProd{u_t^{(b)}}{\widetilde g_t^{(b)}} 
\end{align*}
\end{definition}

%% file: 101-app-related_works.tex
\section{Related Work}
\label{sec: related works}
The goal of privacy auditing is to investigate the extent to which the implementations of private mechanisms adhere to theoretical guarantees of privacy. Membership inference attacks, formalized by Shokri et al. \cite{shokri2017membership} and further analyzed by Yeom et al. \cite{yeom2018privacy}, serve as the fundamental primitive for auditing data privacy in machine learning, quantifying the extent to which a model leaks information about its training set. Carlini et al. \cite{carlini2019secret} introduced the concept of injecting unique, identifiable "canaries" into the training data to measure unintended memorization. 

Building on these foundations, a substantial line of work audits DP-SGD by repeatedly training the mechanism under neighboring datasets and converting attack success rates into privacy lower bounds, beginning with the data-poisoning audits of Jagielski et al. \cite{jagielski2020auditing} and the first tight white-box DP-SGD audit of Nasr et al. \cite{nasr2021adversary}. Subsequent work has refined the statistical conversion via Bayesian credible intervals \cite{zanella2023bayesian}, generic black-box estimators \cite{lu2024eureka, askin2025general}, and worst-case initializations for tight black-box bounds \cite{annamalai2024nearly}. Most relevant to our work, Nasr et al. \cite{nasr2023tight} fit a Gaussian DP model to canary-aligned dot products, but their exact analysis applies to a single, non-subsampled Gaussian step and requires thousands of training runs. We instead prove that the composed, subsampled canary score is asymptotically Gaussian and exploit this in a single training run.


Steinke et al. \cite{steinke2023privacy} introduced the idea that, by parallelizing across many independent canaries, a single training run can produce enough statistical evidence for a meaningful privacy lower bound. Their construction thresholds each canary's score into a binary membership guess and aggregates these guesses into a confidence interval on $\varepsilon$. Mahloujifar et al. \cite{mahloujifar2025auditing} sharpened this analysis by working directly with $f$-DP trade-off curves, which enables auditing of the full privacy trade-off curve, rather than only testing a localized $(\varepsilon,\delta)$ point. Xiang et al. \cite{xiang2025bits,xiang2025tight} reinterpret one-run auditing as a noisy-channel coding problem and use order statistics to select high-confidence guesses, while Liu et al. \cite{liu2025enhancing} strengthen the underlying membership-inference attack via quantile regression. Keinan et al. \cite{keinan2025well} characterize the limits of distribution-free, guessing-based one-run auditing. Our work differs from these works by taking a different analysis route: instead of thresholding the canary score to a binary random variable, we exploit a parametric Gaussian model for real-valued canary scores. We also note that Steinke et al. \cite{steinke2023privacy} observe that their analysis is generic and may be loose for realistic Gaussian-noise mechanisms. They point to algorithm-specific analyses exploiting the iterative structure of DP-SGD as an important future direction, and our work pursues that direction and gives an affirmative answer.


In federated learning, Maddock et al. \cite{maddock2023canife} and Andrew et al. \cite{andrew2024oneshot} studied the problem of empirical privacy estimation in both per-round and one-shot settings. We note that, while Andrew et al. \cite{andrew2024oneshot} also uses a Gaussian approximation, it is applied to cosine-angle statistics between canary directions and noised gradient, rather than to the aggregate canary score. Our Gaussian model instead characterizes the normalized sum of canary observations over the full training trajectory, which is crucial in composed mechanisms such as DP-FTRL and DP-SGD and, in experiments, we show a clear empirical advantage.

%% file: 03-methodology.tex
\section{Methodology}
\label{sec:methodology}

In this section, we present our one-run white-box auditing procedure. Given a fixed privacy parameter $\delta$, the auditor outputs an empirical lower bound on the corresponding privacy parameter $\varepsilon$ for a DP-SGD implementation and can be extended to other mechanisms. At a high level, the procedure has three steps. First, the auditor runs white-box DP-SGD once with many independently inserted canaries, following the transcript model in~\Cref{alg:dpsgd-transcript}. Each canary induces a sequential observation $X_{1:T}^{(b)}$ from the resulting transcript, as defined in~\Cref{def:canary-aligned-obs}. Second, for each canary, the auditor aggregates and normalizes its sequential observation into a single scalar audit score $S_T^{(b)}$. Across many canaries, these scores form two empirical samples: one from the absent-canary world and one from the present-canary world. Third, the auditor models these two samples as independent draws from two Gaussian distributions, and the hockey-stick divergence between these two Gaussians yields the empirical lower bound on $\varepsilon$.

\Cref{def:canary score} introduces the scalar score produced by each canary. Recall that \Cref{def:canary-aligned-obs} associates each canary with a sequential observation
$X_{1:T}^{(b)}$, where $X_t^{(b)}$ is the projection of the noisy gradient at iteration $t$ onto the canary's direction. The auditor compresses this sequence into a single normalized score by summing the projected observations over the training trajectory and scaling by $1/\sqrt{T}$. 

\begin{definition}[Canary score]
\label{def:canary score}
Let $X_{1:T}^{(b)}$ be as in~\Cref{def:canary-aligned-obs}. The \emph{canary score} induced by this sequential observation is
\[
    S_T^{(b)} \defin \frac{1}{\sqrt{T}}\sum_{t=1}^{T} X_t^{(b)} .
\]
\end{definition}

The per-iteration observation $X_t^{(b)}$ admits a natural signal-plus-noise interpretation. Intuitively, when the canary is absent, the projected observations $X_t^{0}$ contain only background contributions and projected DP noise. When the canary is present, $X_t^{1}$ additionally contains a clipped signal if the canary is sampled. Thus, the canary score $X_t^{b}$ is expected to shift between the neighboring worlds $b=0$ and $b=1$ with an amount of Bernoulli-sampled canary contributions. 

This motivates a Gaussian model for the normalized canary score $S_T^{(b)} = \frac{1}{\sqrt{T}}\sum_{t=1}^{T} X_t^{(b)}$. In the world where the canary is absent, the score is dominated by sum of independent DP noise. In the world where the canary is present, the score additionally accumulates Bernoulli-sampled canary contributions. By the Central Limit Theorem, the normalized sum of many i.i.d. finite-variance observations is well approximated by a one-dimensional Gaussian. We therefore approximate the canary-score distributions under the two neighboring worlds by a Gaussian pair
\begin{align*}
    G_0 \defin \multiNormal{\mu_0}{\sigma_0^2},
    \qquad
    G_1 \defin \multiNormal{\mu_1}{\sigma_1^2},
\end{align*}
where $G_0$ models the distribution of $S_T^{(0)}$ and $G_1$ models the distribution of $S_T^{(1)}$.

We next present our modeling choice for the canary-score distribution $S_T^{(b)}$ explicitly. This model isolates the two dominant factors: the projected Gaussian DP noise (present in both worlds) and the clipped canary contribution (present only when the canary is sampled).

\begin{model}[Canary-observation model]
\label{model:canary score}
For each $t\in[T]$, the absent- and present-canary observations take the form
\begin{align*}
    X_t^{(0)} = \rZ_t^{(0)}, 
    \qquad 
    X_t^{(1)} = B_t C + \rZ_t^{(1)},
    \qquad t \in [T],
\end{align*}
where $B_t \stackrel{\mathrm{i.i.d.}}{\sim} \mathrm{Bernoulli}(q)$ and
$Z_t^{(0)}, Z_t^{(1)}
\stackrel{\mathrm{i.i.d.}}{\sim} \mathcal N(0,\sigma^2)$
are mutually independent. 
\end{model}

Under \Cref{model:canary score}, the absent-canary score is exactly Gaussian, while the present-canary score is a normalized sum of independent Bernoulli-Gaussian random variables and is therefore asymptotically Gaussian. \Cref{prop:canary score model} gives the corresponding mean and variance of the Gaussian surrogate in terms of the DP-SGD parameters $C$, $T$, $q$, and $\tau^2$.

\begin{proposition}
\label{prop:canary score model}
Let $S_T^{(0)}$ and $S_T^{(1)}$ be the canary scores (see~\Cref{def:canary score}). Then under \Cref{model:canary score},
\begin{align*}
    S_T^{(0)} \sim \mathcal N(0,\sigma^2), \qquad
    \mathbb{E}[S_T^{(1)}] = \sqrt{T}qC, \qquad
    \mathrm{Var}(S_T^{(1)}) = \sigma^2 + q(1-q)C^2
\end{align*}
Moreover, $S_T^{(1)}$ asymptotically follows the distribution $\multiNormal{\sqrt{T}\,qC}{\sigma^2 + q(1-q)C^2}$ as $T \rightarrow \infty$.
\end{proposition}

\paragraph{From adaptive DP-SGD trajectory to Model~\ref{model:canary score}.}
In a real DP-SGD trajectory, canary
magnitude, and canary direction can depend on the training history. Appendix~\ref{app:dependent-dpsgd} gives the
corresponding martingale decomposition and establishes the following reduction.

\begin{proposition}[Reduction of adaptive scores to Model~1]
\label{prop:adaptive-reduction}
Consider the adaptive DP-SGD score process of
Appendix~\ref{app:dependent-dpsgd}. Suppose that (i) background gradients have
zero projection onto the canary direction at every step, and (ii) the
present-canary clipped gradient has norm $C$ at every step. Then the adaptive score process has exactly the formulation of Model \ref{model:canary score}.
\end{proposition}

A fixed-coordinate Dirac gradient canary enforces exact norm saturation $C$ by construction, and we show that background suppression holds empirically in Section \ref{sec: bg-suppression}. The formal statements and proofs, including the adaptive martingale decomposition, are deferred to Appendices~\ref{app:dependent-dpsgd} and \ref{app:missing-proofs-dependent-dpsgd}.


\begin{algorithm}[htp]
\caption{One-run white-box DP-SGD auditor (under Model~\ref{model:canary score})}
\label{alg:auditor}
\begin{algorithmic}[1]
\Require Privacy parameter $\delta$; absent-canary scores $\cS_0=\{s_i^{(0)}\}_{i=1}^{m_0}$; present-canary scores $\cS_1=\{s_j^{(1)}\}_{j=1}^{m_1}$; parameter $m = m_1 + m_0$; confidence level $\alpha$.
\Ensure Empirical lower bound $\widehat\varepsilon_{\lb}$ for privacy parameter $\varepsilon$.
\State Using $\cS_0$ and $\cS_1$, construct a $(1-\alpha)$-confidence region
\begin{align*}
    C_\alpha \subseteq \Theta
    \qquad\text{for}\qquad
    \theta^\star=(\mu_0,\sigma_0,\mu_1,\sigma_1),
\end{align*}
where $\multiNormal{\mu_0}{\sigma_0^2}$ and $\multiNormal{\mu_1}{\sigma_1^2}$ are the absent- and present-canary score distributions.
\smallskip
\State Compute $\widehat\varepsilon_{\lb}(\alpha) \gets \inf_{\theta\in C_\alpha} \varepsilon_\theta(\delta),$
where $\varepsilon_\theta(\delta)$ is defined in~\Cref{def:epsilon-func-delta}.
\State \Return $\widehat\varepsilon_{\lb}(\alpha)$.
\end{algorithmic}
\end{algorithm}

The above discussion reduces the one-run white-box auditing problem to estimating the parameters of a one-dimensional Gaussian pair. \Cref{alg:auditor} presents our one-run white-box DP-SGD auditor based on this reduction. Given the absent- and present-canary score samples, the auditor first constructs a $(1-\alpha)$-confidence region $C_\alpha$ for the unknown Gaussian parameters $\theta^\star=(\mu_0,\sigma_0,\mu_1,\sigma_1)$. It then evaluates the most conservative privacy lower bound over all Gaussian pairs in this confidence region by computing
$    \widehat\varepsilon_{\lb}(\alpha) \defin \inf_{\theta\in C_\alpha}\varepsilon_\theta(\delta)$. 
Consequently, the auditor returns $\widehat\varepsilon_{\lb}(\alpha)$ as a confidence-adjusted lower bound of the true privacy parameter $\varepsilon$ under~\Cref{model:canary score}.

%% file: 04-convergence-analysis.tex
\subsection{Convergence Rate of the Gaussian Approximation}
\label{sec:rate-of-convergence}

The methodology in \Cref{sec:methodology} models the canary score $S_T^{(b)}$ as a one-dimensional Gaussian under both the absent-canary and present-canary worlds. Under the idealized scalar observation, the absent-canary score $S_T^{(0)}\sim\multiNormal{0}{\sigma^2}$ is exactly Gaussian for every $T$. The present-canary score $S_T^{(1)}$ however, is only asymptotically Gaussian: each summand $X_t^{(1)} = B_t C + \rZ_t$ mixes a Bernoulli signal with Gaussian noise. In Appendix \ref{app:converge} we present a general theorem that captures the exact rate of convergence for normalized i.i.d.\ sum of Gaussian-mixture random variables. In this subsection we specifically apply it to the DP auditing setting to compute the rate of convergence of $S_T^{(1)}$.

In the DP auditing setting, the per-step audit observation is i.i.d.: $X_t^{(1)} = B_tC + \rZ_t$ with $B_t\iid\ber{q}$ and $\rZ_t\iid\multiNormal{0}{\sigma^2}$. Centering by the mean, $\widetilde X_t \defin X_t^{(1)} - qC$, we can write 
\[
\widetilde X_t\sim\GM{1-q}{q}{-qC}{(1-q)C}{\sigma}
\]
where $\GM{u}{v}{\mu_1}{\mu_2}{\sigma}$ represents a Gaussian mixture distribution with weights $u$ and $v$ (where $u+v=1$) and the two Gaussians have same standard deviation $\sigma$ and different means $\mu_1$ and $\mu_2$ respectively.
We can further define:
\[
\auditStat \defin \frac{1}{\sqrt T}\sum_{t=1}^T \widetilde X_t = S_T^{(1)} - \sqrt T\, qC
\]

For the above normalized sum of Gaussian mixture distribution we derive the exact cumulants and hence its rate of convergence in Appendix \ref{app:converge}, giving us the following result:

\begin{theorem}[Closed-form Kolmogorov-distance bound for the DP audit statistic]
\label{thm:dp-kolmogorov-bound}
Consider the canary-observation \cref{model:canary score} of Algorithm~\ref{alg:auditor} with $T$ training steps, sampling rate $q\in(0,1)$, clipping norm $C>0$, and per-step noise standard deviation $\sigma>0$. Let $\trueLaw$ denote the true law of the centered audit statistic $\auditStat$ and $\modelLaw\defin\multiNormal{0}{\sigma^2 + q(1-q)C^2}$ denote its Gaussian model. Then
\[
    \dK{\trueLaw}{\modelLaw} =
    \frac{q(1-q)\,|1-2q|\,C^3}
         {6\sqrt{2\pi\,T}\,\bigl[\sigma^2 + q(1-q)C^2\bigr]^{3/2}}
    + O(1/T).
\]
\end{theorem}

While in general the audit statistic distribution converges to a Gaussian at rate $O(1/\sqrt{T})$, for asymptotically small sampling rate (e.g. $q = O(1/\sqrt{T})$), as is the case for DP-SGD, the distributions converge even faster (rate at least $O(1/T)$).

\paragraph{Concrete estimates.} To assess the accuracy of the Gaussian approximation in our experiments, Appendix~\ref{app:converge} derives
a tail-localized refinement of \Cref{thm:dp-kolmogorov-bound}
(see \Cref{cor:tail-deviation}).
For our experimental parameters
$T=2500$, $q=0.0819$, $C=1$, $\sigma=2.6245$, the leading-order CDF-deviation estimate is $2.29\times10^{-8}$, which is three orders of magnitude smaller than our audit parameter $\delta=10^{-5}$.
These estimates give additional confidence in the Gaussian approximation for the empirical privacy lower bound in our experimental setup.

%% file: 05-CI-gaussian-estimation.tex
\subsection{Gaussian parameter estimation and confidence regions}
\label{subsec:gaussian-parameter-confidence}

The convergence rate study in~\Cref{sec:rate-of-convergence} shows the Gaussian approximation closely fits the finite-$T$ present-canary score distribution in the DP-SGD settings. In this section, we address a separate source of error: the finite number of canary scores used to estimate the Gaussian parameters. 

\Cref{prop:auditor-confidence-guarantee} states that our auditor's output (\Cref{alg:auditor}) serves as a conservative lower bound of privacy parameter $\varepsilon$ with high probability under the Gaussian score model. 
\begin{proposition}[Confidence guarantee for \Cref{alg:auditor}]
\label{prop:auditor-confidence-guarantee}
Let $\widehat{\varepsilon}_{\lb}(\alpha)$ be the output of \Cref{alg:auditor}. Then
\begin{align*}
    \pr{\widehat{\varepsilon}_{\lb}(\alpha) \leq \varepsilon^\star} \geq 1-\alpha .
\end{align*}
\end{proposition}

This directly follows from the fact that on event $\{\theta^\star\in C_\alpha\}$, the output of \Cref{alg:auditor} satisfies
$     \widehat{\varepsilon}_{\lb}(\alpha) = \inf_{\theta\in C_\alpha}\varepsilon_\theta(\delta)
    \leq \varepsilon_{\theta^\star}(\delta) = \varepsilon^\star .$ 

\noindent\textbf{Hypothesis test.} Equivalently, \Cref{alg:auditor} tests $H_0:\varepsilon\le\varepsilon_{\mathrm{claim}}$ vs.\ $H_1:\varepsilon>\varepsilon_{\mathrm{claim}}$, rejecting $H_0$ when $\widehat{\varepsilon}_{\lb}(\alpha)>\varepsilon_{\mathrm{claim}}$. Since $\varepsilon^\star\le\varepsilon_{\mathrm{claim}}$ under $H_0$, \Cref{prop:auditor-confidence-guarantee} gives $\pr{\widehat{\varepsilon}_{\lb}(\alpha)>\varepsilon_{\mathrm{claim}}}\le\alpha$, and so the test has level $\alpha$.

\noindent\textbf{Constructing $C_\alpha$.} In our implementation, $C_\alpha$ can be instantiated using standard confidence-region constructions for Gaussian parameters. A simple finite-sample choice is the Bonferroni rectangle obtained from the usual Student-$t$ confidence intervals for the means and chi-squared confidence intervals for the variances, applied separately to the absent- and present-canary score samples~\cite{casella2002statistical}. This construction is valid under the Gaussian score model, but can be conservative because it treats the four coordinates of $\theta^\star$ separately.

As a less conservative alternative, one can use a bootstrap ellipsoid based on the joint covariance of the estimator $ 
    \widehat\theta = (\widehat\mu_0,\widehat\sigma_0,\widehat\mu_1,\widehat\sigma_1)$,
where the covariance is estimated from bootstrap resamples of the canary scores~\cite{efron1994introduction}. This ellipsoidal region captures correlations among the estimated Gaussian parameters and has asymptotic coverage under standard bootstrap regularity conditions. 

%% file: 06-experiments.tex
\section{Experiments and Results}\label{sec:exp}
We empirically validate our auditing method 
on composition DP mechanisms where the audit statistic converges to a Gaussian distribution. We first audit DP-SGD with Poisson subsampling and analytic-Gaussian accounting, showing that our distributional perspective substantially improves auditing power and enables detection of privacy violations that previous auditors miss. We then apply our method to DP-FTRL and DP-MF-FTRL, which use correlated noise. We ran our experiments on a shared HPC cluster utilizing NVIDIA H200-100GB GPUs. Our code is available \href{https://anonymous.4open.science/r/gauss-privacy-audit/}{here.}

\subsection{Auditing for DP-SGD}
To verify the tightness of our estimated lower bounds, we follow the white-box auditing protocol established in Steinke et al. \cite{steinke2023privacy} and Mahloujifar et al. \cite{mahloujifar2025auditing}. We train a WideResNet-16-4 on CIFAR-10 (Batch size $B=4096$, Augmentation $K=16$, $T=2500$ steps). At a theoretical $\varepsilon = 8$, the model achieves $\approx 70\%$ accuracy with a negligible utility drop ($\approx 1\%$) upon the inclusion of 5,000 canaries. We follow the auditing procedure described in Algorithm \ref{alg:auditor}. We evaluate all three methods at a 95\% confidence level. While baselines suffer Bonferroni penalties from threshold searching, our method avoids thresholding entirely by extracting worst-case parameters from a continuous 95\% bootstrap ellipsoid (See Sec ~\ref{subsec:gaussian-parameter-confidence}). As illustrated in Figure \ref{fig:combined-auditing-plots}, our method yields significantly tighter lower bounds than existing one-run baselines (by $1-2\times$) across all $\varepsilon$ regimes. Our method narrows the gap to the analytic upper bound, capturing $63\%$--$84\%$ of the theoretical limit. Crucially, these gains are achieved without modifying the underlying training or auditing pipeline, simply by leveraging the distributional information inherent in the sequential gradient observations.
\begin{figure}[t]
    \centering
    \begin{subfigure}{0.49\textwidth}
        \centering
        \includegraphics[width=\linewidth]{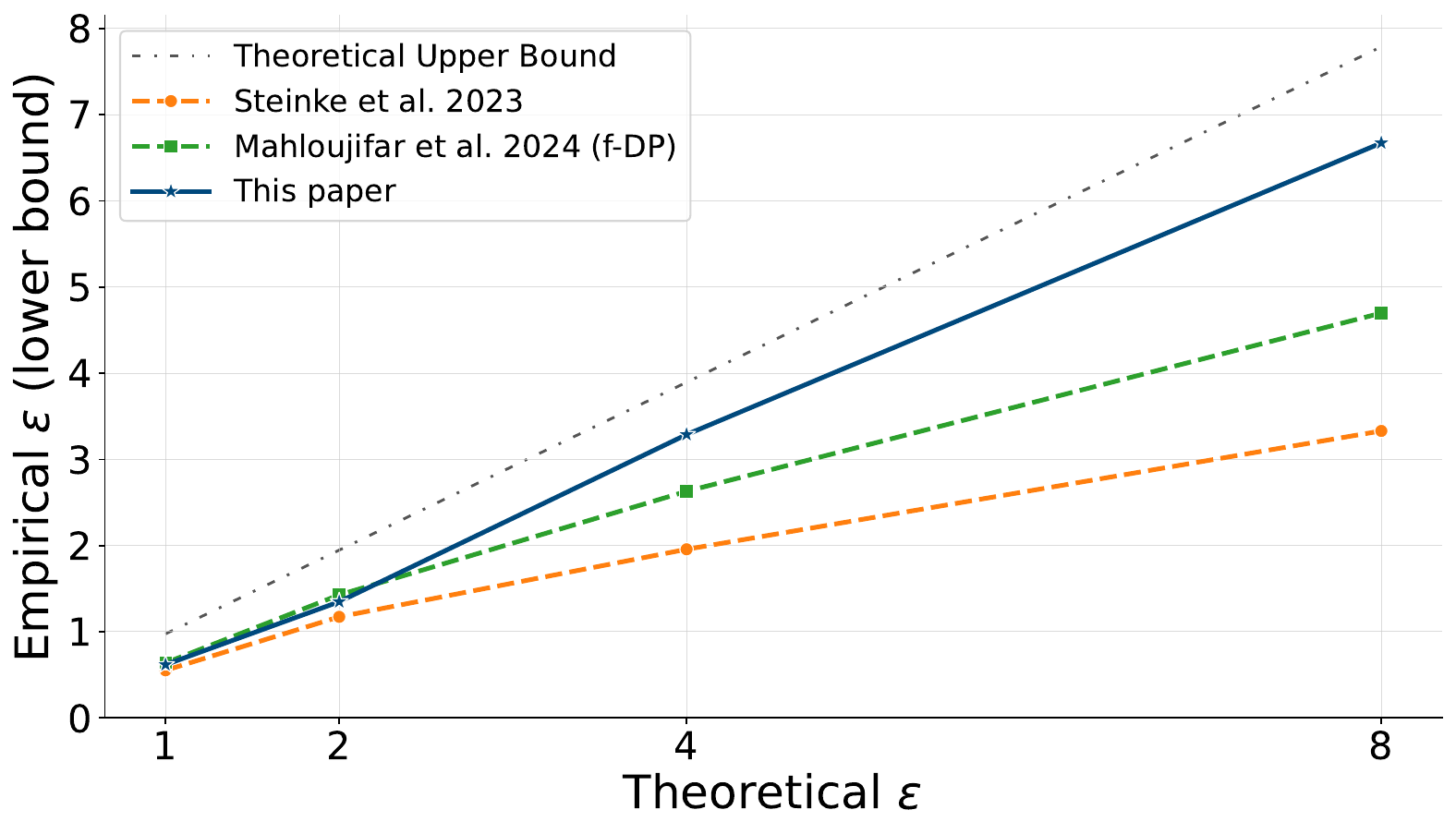}
        \label{fig:comparision-sweeping-eps}
    \end{subfigure}
    \hfill 
    \begin{subfigure}{0.45\textwidth}
        \centering
        \includegraphics[width=\linewidth]{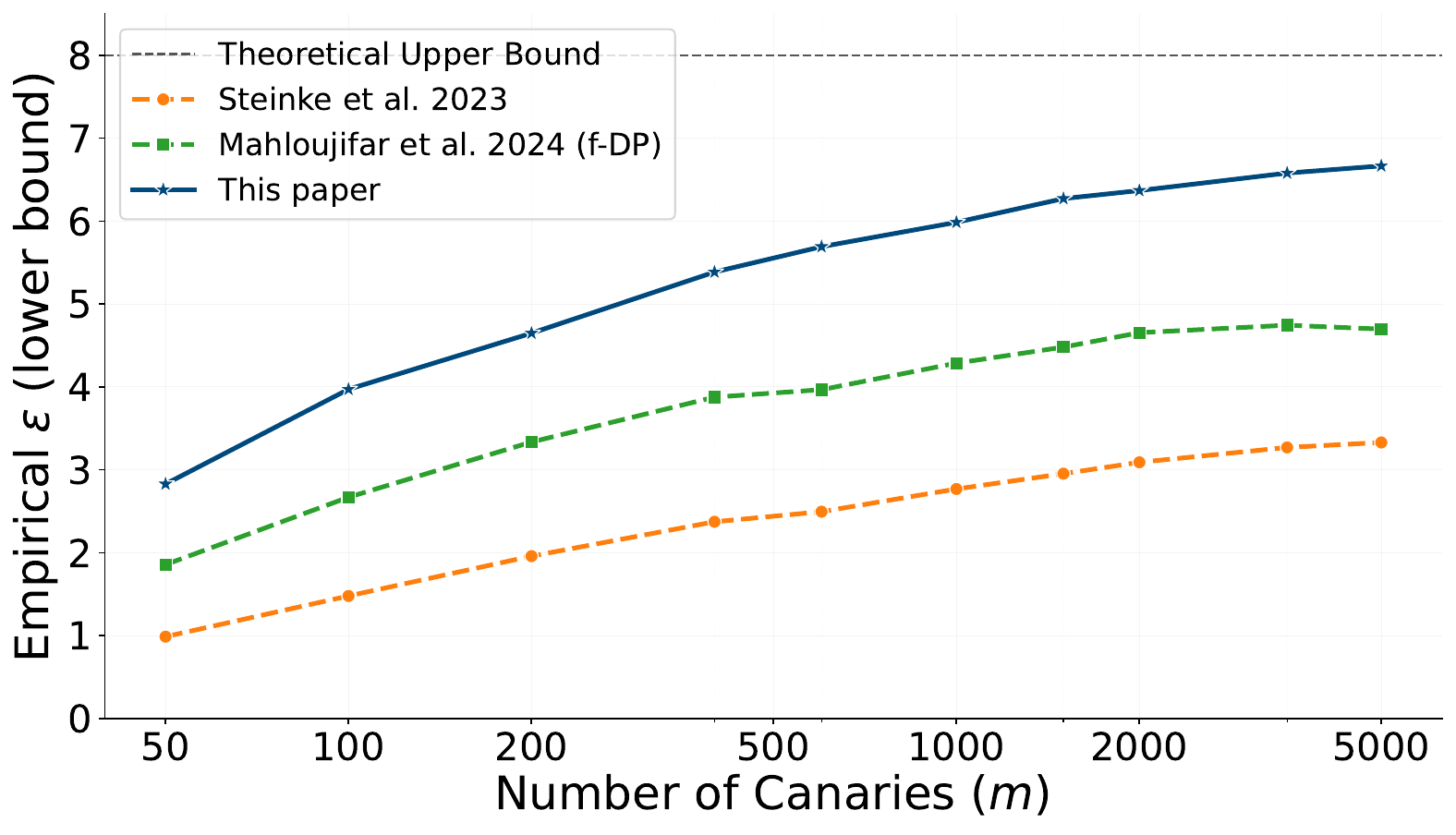}
        \label{fig:canary-complexity-sgd}
    \end{subfigure}
    
    \caption{DP-SGD auditing results: (a) Left: shows empirical lower bounds across different theoretical $\epsilon$, and (b) Right: shows the impact of canary count $m$ on auditing tightness compared to Steinke et al. \cite{steinke2023privacy} and Mahloujifar et al. \cite{mahloujifar2025auditing}. All lower bounds are computed at a 95\% confidence.}
    \label{fig:combined-auditing-plots}
\end{figure}

We also vary the number of canaries, starting at $m=50$. As shown in Figure~\ref{fig:combined-auditing-plots}, all three methods are relatively stable beyond $m=1000$, but the baselines converge to significantly looser bounds. Our method recovers a larger fraction of the theoretical privacy budget, obtaining tighter bounds with only 100 canaries than competing methods achieve with nearly 2500. Appendix~\ref{app:adaptive-clipping-app} also extends our auditor to DP-SGD with adaptive clipping.

\begin{table}[H]
\centering
\caption{Detecting privacy violations incurred due to DP-SGD implementation errors}\label{tab:impl_errors}
\label{tab:bugs}
\footnotesize
\setlength{\tabcolsep}{5pt}
\begin{tabularx}{\linewidth}{@{}>{\raggedright\arraybackslash}X r rrr@{}}
\toprule
Implementation errors & True $\varepsilon$
  & Steinke et al. \cite{steinke2023privacy}
  & Mahloujifar et al. \cite{mahloujifar2025auditing}
  & Ours \\
\midrule
Correct implementation
  & 8.00  & 3.330 & 4.697 & 6.651 \\
\midrule
Noise scaled for clipping bound $C/2$, trained with $C$
  & 21.98 & 5.542 & \textbf{9.404} & \textbf{15.835} \\
\addlinespace
Training steps doubled, noise not recalibrated
  & 12.12 & 4.315 & 6.654 & \textbf{9.778} \\
\addlinespace
Sampling rate $2\times$ the accounted rate
  & 18.88 & 5.528 & \textbf{9.365} & \textbf{16.162} \\
\bottomrule
\end{tabularx}
\end{table}

\subsection{Detecting privacy violations for implementation errors}
We evaluate our auditor and the baselines on three synthetic implementation errors in DP-SGD: a clipping-sensitivity mismatch, incorrect noise calibration, and a sub-sampling rate mismatch. These errors are modeled after documented failures in practice \cite{tramer2022debugging, cebere2026privacy, pmlr-v235-chua24a}.

The errors substantially increase the true privacy loss, from the claimed $\epsilon = 8$ to $12.12-21.98$ as shown in Table \ref{tab:impl_errors}. Our tighter lower bound exceeds the claimed bound for all three tested errors and therefore detects each violation in this experiment. Steinke et al. \cite{steinke2023privacy} detects none, while the auditor of Mahloujifar et al. \cite{mahloujifar2025auditing} detects two of the three. These results demonstrate that common calibration and accounting errors can change the parameters of the score distributions without destroying their Gaussian shape, allowing our auditor to audit them with greater power than existing auditors.

\subsection{Auditing for applications beyond DP-SGD}
Our framework is not tied to DP-SGD as a specific algorithm. It applies whenever the auditor can construct an audit statistic whose distributions under the neighboring worlds are well approximated by a pair of Gaussians. Matrix-factorization
(MF) mechanisms, which include DP-FTRL, can be audited with Gaussian statistics.

Let
$x\in\mathbb{R}^{T\times d}$ stack the per-round sums of clipped gradients
and let $A$ be the prefix-sum workload. An MF mechanism factors
$A=\mathbf{B}\mathbf{C}$, releases $y=\mathbf{C}x+Z$ with i.i.d.\
$\mathcal{N}(0,\sigma^2)$ entries in $Z$, and decodes
$\mathbf{B}y=Ax+\mathbf{B}Z$. 

\begin{definition}[Canary score for MF] Let canary $i$ be a Dirac gradient canary on coordinate
$c_i$, inserted at round $t_i$. We use the likelihood-ratio score
$ S_i =
\langle \mathbf{C}_{:,t_i},\, y_{:,c_i}\rangle / \|\mathbf{C}_{:,t_i}\|_2.$ 
\end{definition}
Under Model \ref{model:canary score}, coordinate $c_i$ of $y$ is
$\mathcal{N}(0,\sigma^2 I)$ when the canary is absent and
$\mathcal{N}(C\,\mathbf{C}_{:,t_i},\sigma^2 I)$ when it is present. Hence, the Gaussian convergence is exact with no approximation errors. 

The tree-based DP-FTRL of Kairouz et al. 
\cite{kairouz2021practical} is the special case in which $\mathbf{C}$ is the binary-tree encoder, whereas for DP-MF-FTRL the encoder comes from a factorization
$\mathbf{A}=\mathbf{B}\mathbf{C}$ of the prefix-sum workload. In Figures~\ref{fig:dp-ftrlmulti-eps} and \ref{fig:dp-mf-ftrlmulti-eps}, we compare our auditor against the one-shot auditing method of Andrew et al. \cite{andrew2024oneshot}. While their approach also utilizes a closed form $(\epsilon, \delta)$ estimator using Gaussian parameters, it relies on a maximum-score statistic over training steps. In contrast, our framework leverages the entire joint distribution of sequential observations. Details of the audit statistics, gaussian pairs and experimental setup are in Appendix \ref{app:ftrl}. 

\begin{figure}[htbp]
  \begin{minipage}{0.49\textwidth}
    \centering
    \includegraphics[width=\linewidth]{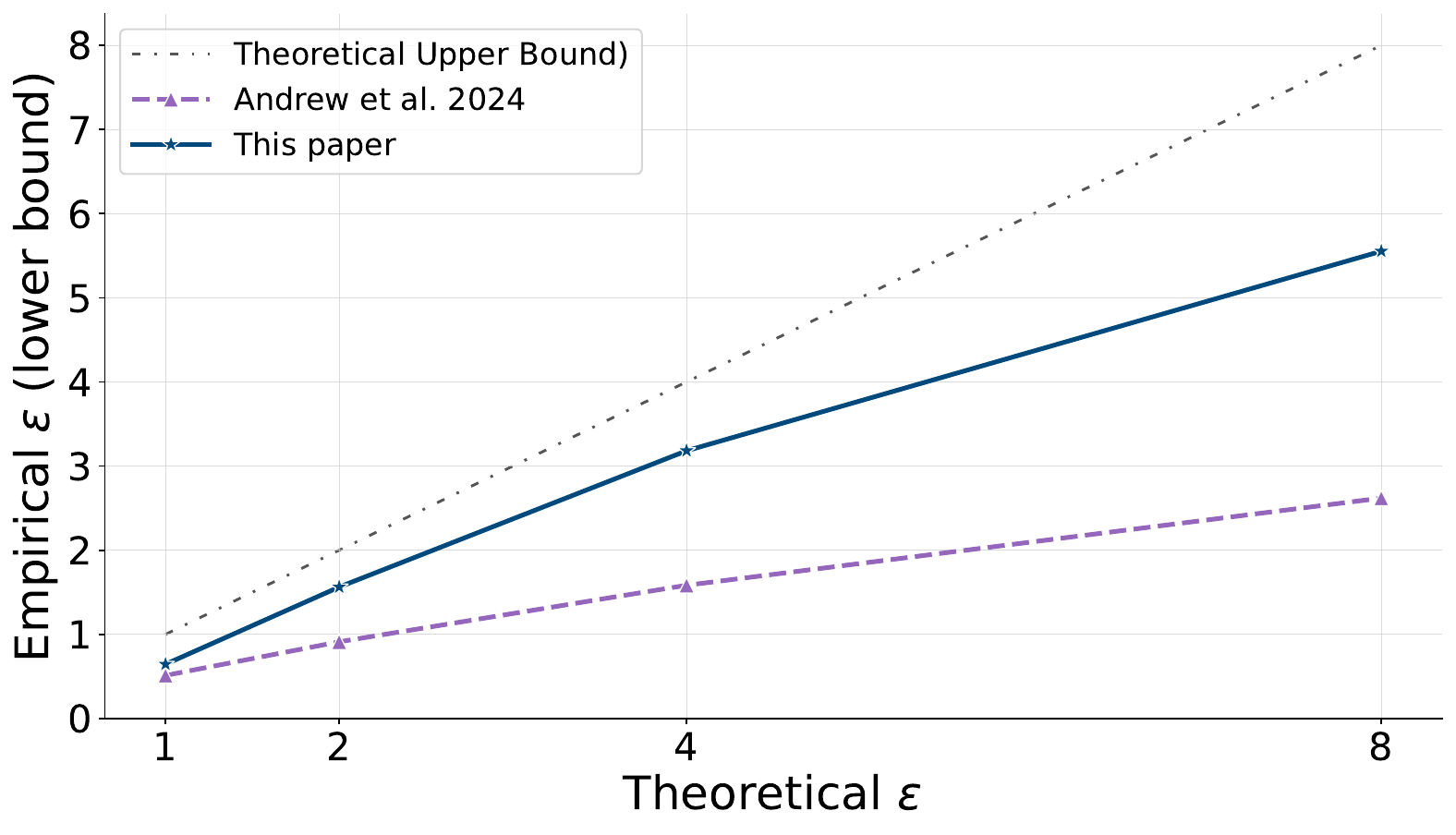}
    \caption{Auditing results for DP-FTRL across multiple $\varepsilon$ values compared to Andrew et al. \cite{andrew2024oneshot}}
    \label{fig:dp-ftrlmulti-eps}
  \end{minipage}
  \hfill 
  \begin{minipage}{0.49\textwidth}
    \centering
    \includegraphics[width=\linewidth]{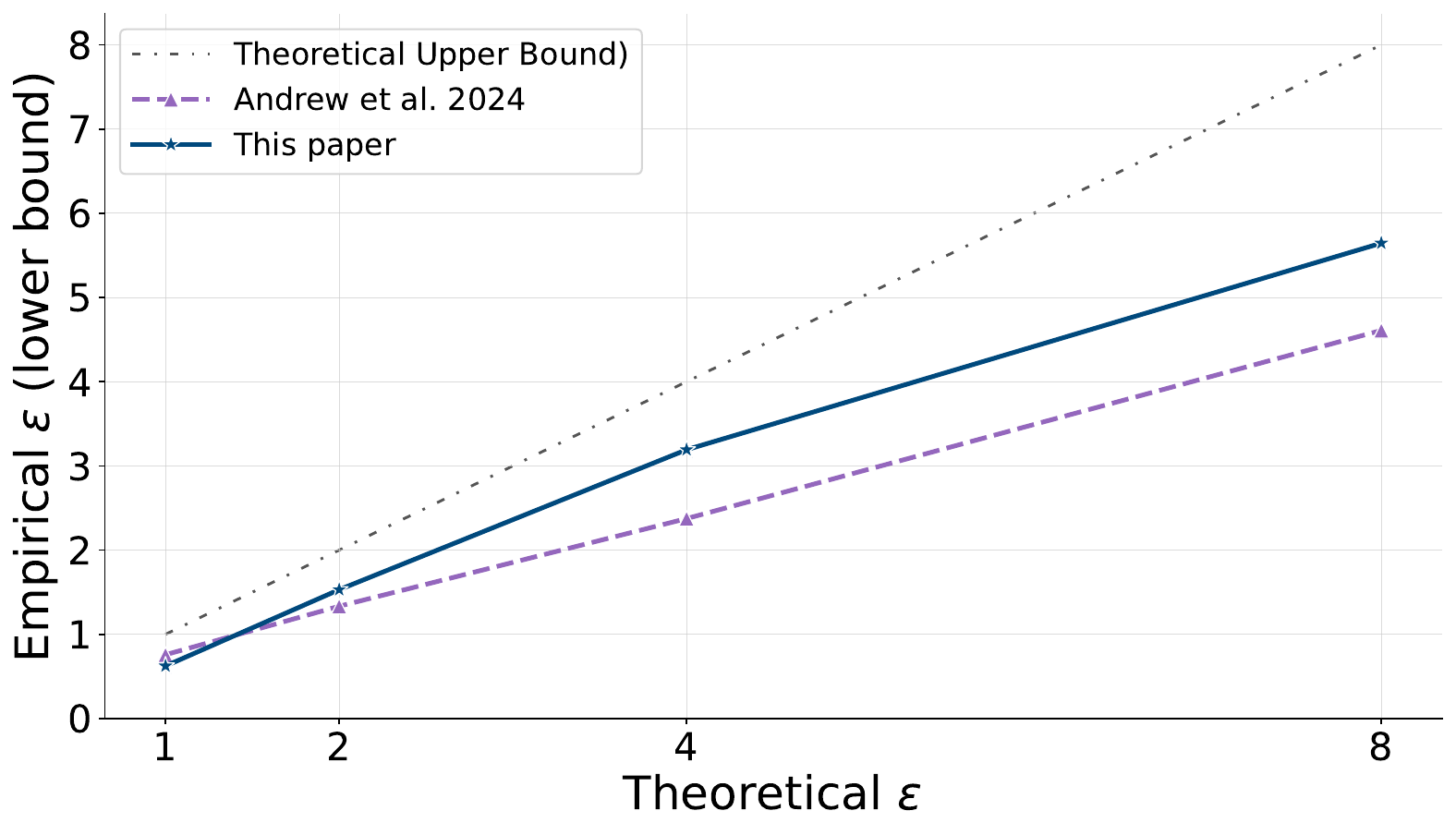}
    \caption{Auditing results for DP-MF-FTRL across multiple $\varepsilon$ values compared to Andrew et al. \cite{andrew2024oneshot}}
    \label{fig:dp-mf-ftrlmulti-eps}
  \end{minipage}
\end{figure}

\subsection{Ablations}

We run additional ablations to test our model and robustness across the tested hyperparameter ranges, and the observed behavior agrees with the dependence predicted by Proposition~\ref{prop:canary score model}. We add detailed experiments varying the clipping bound, batch size and learning rate in Appendix \ref{app:additional-ablations}.

\subsubsection{Background suppression}\label{sec: bg-suppression}
Our Model \ref{model:canary score} assumes each canary score is only DP noise plus the canary's own
signal, with no contribution from the real training records at the canary coordinate.
To test this, we compare the real audit statistic from a  DP-SGD run against 500 background-free simulations. The simulations use the same $(\sigma, C, T, q)$ and the same presence split. Figure~\ref{fig:qq} compares the scores from the absent canary world, divided by the DP noise standard deviation $\sigma C$, with the background-free prediction $\mathcal{N}(0,1)$. At every $\varepsilon$ the quantiles lie on the line of perfect agreement; the only deviations are in the extreme tails, where the number of samples is low in both directions. This indicates that any background contribution is therefore below sampling noise and does not affect the audit. We also add a hexbin density against canary index plot in Figure~\ref{fig:hexbin} of Appendix~\ref{app:additional-ablations}.

\begin{figure}[t]
    \centering
    \begin{minipage}
        {0.49\linewidth}
        \centering
        \includegraphics[width=\linewidth]{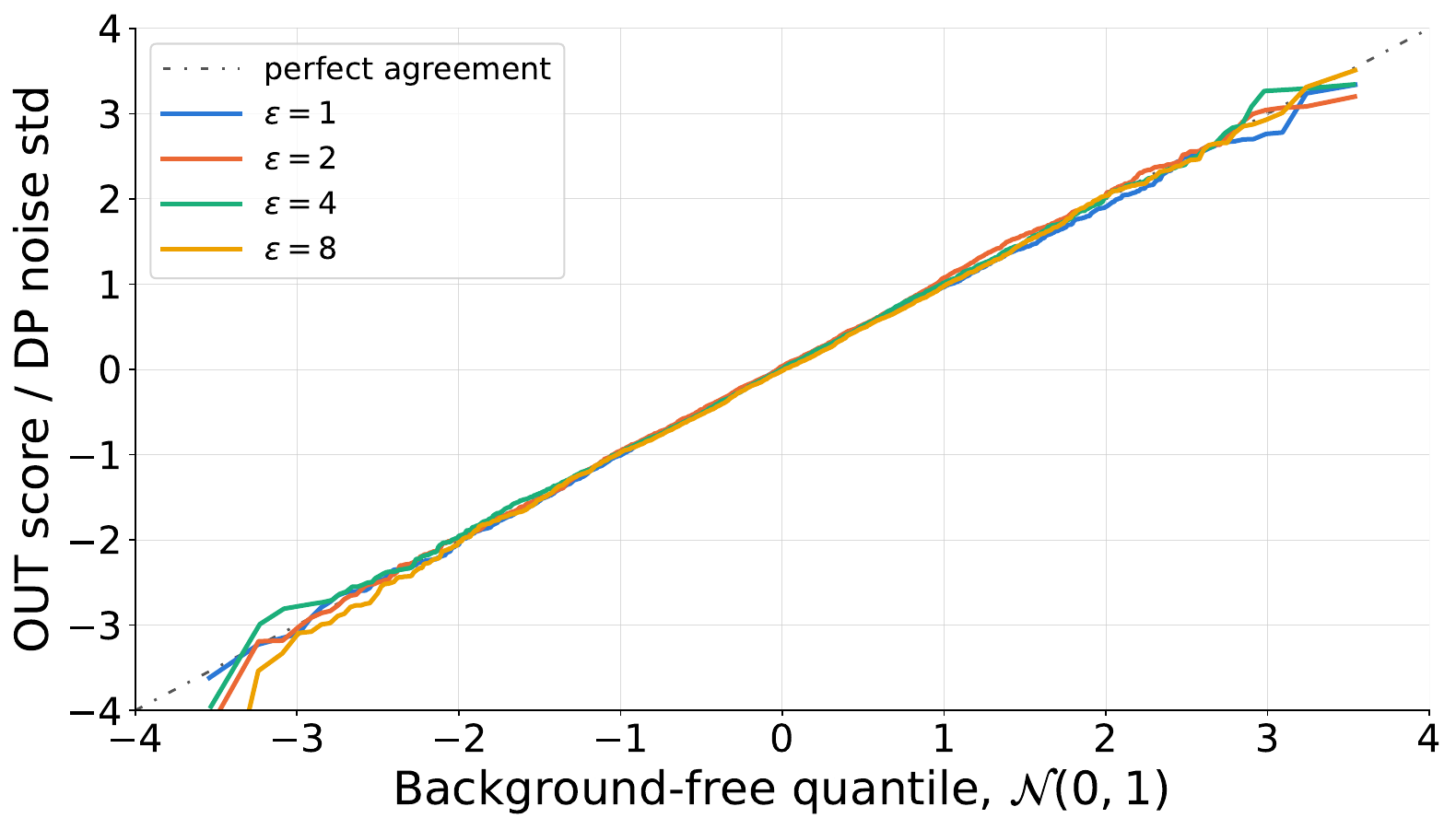}
        \caption{Q-Q plot of real canary scores against the background-free prediction.}
        \label{fig:qq}
    \end{minipage}
    \hfill
    \begin{minipage}{0.49\linewidth}
        \centering
        \includegraphics[width=\linewidth]{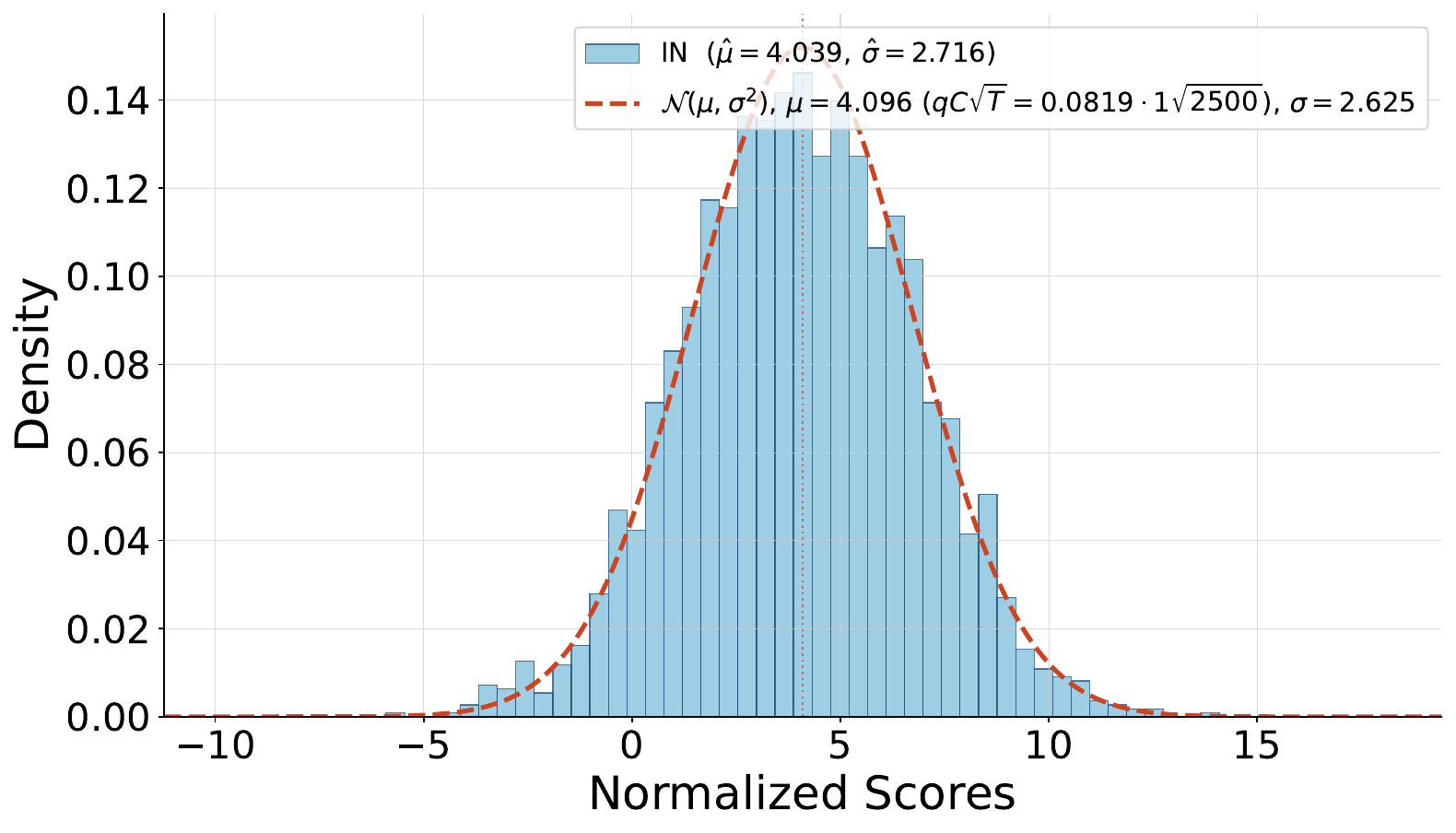}
        \caption{Histogram to visualize Gaussianity of canary-present scores against analytical parameters.}
        \label{fig:hist-in}
    \end{minipage}
\end{figure}

\subsubsection{Convergence and Geometry in DP-SGD}
\textbf{Gaussian Approximation} In Figure \ref{fig:hist-in}, we see that the histogram of distribution of estimated canary-in scores compared against the analytical Gaussian parameters is very close. We also evaluate the stability of the bound across training steps $T \in \{50, 100, \dots, 2500\}$ in Figure~\ref{fig:ablation_T_eps8} of Appendix~\ref{app:additional-ablations}. The empirical bound remains robust even at low step counts, confirming that the Gaussian convergence of the gradient sums happens rapidly in the DP-SGD regime. 

\textbf{Confidence Region Geometry} We compare the Parametric Bonferroni and Bootstrap Ellipsoid confidence regions. As shown in the Appendix \ref{app:additional-ablations}, the Ellipsoid consistently recovers higher $\varepsilon$ values by capturing the cross-correlation between sample means and variances, which the axis-aligned Bonferroni rectangle ignores.

%% file: 07-discussion.tex
\section{Discussion}
\label{ref:discussion}

We presented a one-run, white-box auditor that departs from the binary-thresholding paradigm of prior work and instead it uses the full distributional information in canary-aligned observations. Our key insight is that for DP-SGD, after normalization the canary score distribution is asymptotically Gaussian, which lets us convert these scores directly into a privacy lower bound via the hockey-stick divergence between a Gaussian pair. We also extend to other composition algorithms and our experiments show $1-2\times$ tighter lower bounds than existing one-run baselines. 

Our auditing framework uses a distributional perspective on the canary scores, whereas threshold-based one-run audits \cite{steinke2023privacy,mahloujifar2025auditing}
are distribution-free. As noted by Steinke et al. \cite{steinke2023privacy}, exploiting the structure
of a specific mechanism can yield stronger auditing results. We show that our gaussian perspective improves auditing power significantly. We hope this distributional perspective serves as a useful building block for future auditing schemes in settings beyond those discussed here, leading to meaningful ways to detect privacy violations in practice.

%% file: 08-ack.tex
\section{Acknowledgements}
The authors were supported in part by NSF Award No. 2531010, Halcyon Futures via the AI Security Institute, JPMorgan Chase, and by Sunday Group, Inc. This research was supported in part through research cyberinfrastructure resources and services provided by the Partnership for an Advanced Computing Environment (PACE) at the Georgia Institute of Technology, Atlanta, Georgia, USA. 

%% file: 100-app.tex
\begin{appendices}
\input{102-app-transcript.tex}
\input{105-app-dependent-dpsgd}

\input{106-app-missing-proofs-dependent-dpsgd}

\input{103-app-convergence.tex}

\input{107-experimental-details}
\input{104-app-ablations.tex}
\end{appendices}

%% file: 102-app-transcript.tex
\section{White-box DP-SGD Transcript}
\label{app:transcript}
\begin{algorithm}[htp]
\caption{White-box DP-SGD transcript, following the clipped-and-noised DP-SGD update of~\cite{AbadiCGMMT016}.}
\label{alg:dpsgd-transcript}
\begin{algorithmic}[1]
\Require Background database $\DB$, target canary $x^{\star}$, bit $b\in\bits$, loss function $\ell$, clipping norm $C$, subsampling rate $q$, noise multiplier $\sigma$, step sizes $(\eta_t)_{t=1}^{T}$, initial parameter $\theta_0$, and number of iterations $T$.
\Ensure White-box transcript $\mathsf{Tr}^{(b)}_T$.
\State Set $\DB^{(0)} \defin \DB$, $\DB^{(1)} \defin \DB \cup \{x^{\star}\}$, $n_b \defin |\DB^{(b)}|$.
\State Initialize $\theta_0^{(b)} \gets \theta_0$ and $\mathsf{Tr}^{(b)}_0 \gets (\theta_0^{(b)})$.
\For{$t=1,2,\ldots,T$}
    \For{each record $i\in[n_b]$}
        \State Sample $B_{t,i}^{(b)}\sim\ber{q}$ independently.
    \EndFor
    \State Set the minibatch $\cL_t^{(b)} \gets \{\DB^{(b)}[i] : B_{t,i}^{(b)}=1,\ i\in[n_b]\}.$
    \State Compute the noisy averaged gradient
    \begin{align*}
        \widetilde g_t^{(b)}
        = \frac{1}{q n_b}
        \left(
            \sum_{z\in\cL_t^{(b)}}
            \clip_C\left(\nabla_{\theta}\ell(\theta_{t-1}^{(b)};z)\right)
            + \rZ_t^{(b)}
        \right),
        \qquad
        \rZ_t^{(b)}\sim\multiNormal{0}{\sigma^2 C^2 I_d}.
    \end{align*}
    \State Update $\theta_t^{(b)} \gets \theta_{t-1}^{(b)} - \eta_t\widetilde g_t^{(b)}.$
    \State Append $(\widetilde g_t^{(b)},\theta_t^{(b)})$ to $\mathsf{Tr}^{(b)}_{t-1}$.
\EndFor
\State \Return $\mathsf{Tr}^{(b)}_T$.
\end{algorithmic}
\end{algorithm}

%% file: 105-app-dependent-dpsgd.tex
\section{Adaptive DP-SGD Scores: A Martingale Extension}
\label{app:dependent-dpsgd}

Recall from \Cref{alg:dpsgd-transcript,def:canary-aligned-obs} that $b\in\{0,1\}$ indicates whether the canary is absent or present, $q\in(0,1)$ is the Poisson-subsampling rate, $\theta_{t-1}^{(b)}$ is the model iterate after step $t-1$, and $u_t^{(b)}$ is the unit direction induced by its clipped canary gradient. Fix a neighboring world $b$ and a rate $q$. Let the pre-sampling sigma-field $\cF_{t-1}^{(b)}$ contain the fixed data and canary, all training randomness through step $t-1$, and any auxiliary randomness used to form the candidate per-record gradients at step $t$ (such as data augmentations or adaptive clipping randomness). Thus $\theta_{t-1}^{(b)}$, $u_t^{(b)}$, and all quantities defined below are $\cF_{t-1}^{(b)}$-measurable. Let $\cF_t^{(b)}$ additionally reveal the step-$t$ sampling indicators $I_{t,z}^{(b)}$ for $z\in\DB$ and $B_t^{(b)}$ for the canary $x^*$, together with the fresh Gaussian noise. For a background record
$z\in\DB$, write
\begin{align*}
    g_{t,z}^{(b)} \defin \clip_C\!\left(\nabla_\theta\ell(\theta_{t-1}^{(b)};z)\right),
    \qquad h_{t,z}^{(b)} \defin\innerProd{u_t^{(b)}}{g_{t,z}^{(b)}},
    \qquad a_t^{(b)} \defin \twoNorm{c_t^{(b)}} .
\end{align*}
Thus $|h_{t,z}^{(b)}|\le C$ and $0\le a_t^{(b)}\le C$. Let $I_{t,z}^{(b)}\sim\ber{q}$ be the Poisson-sampling indicator of background record $z$, and let $B_t^{(b)}\sim\ber{q}$ be the indicator for the canary. Conditional on $\cF_{t-1}^{(b)}$, these indicators and the fresh Gaussian mechanism noise are mutually independent.

\begin{proposition}[Proof is in Appendix~\ref{app:proof-exact-adaptive-decomposition}] \label{prop:exact-adaptive-decomposition}
For every $t\in[T]$ and $b\in\{0,1\}$, the canary observation takes the form
\begin{align*}
   X_t^{(b)} =m_t^{(b)}+\xi_t^{(b)},
\end{align*}
where
\begin{align*}
  &{}m_t^{(b)} =q\!\left[\sum_{z\in\DB}h_{t,z}^{(b)}+b\,a_t^{(b)} \right], \\
  &{}\xi_t^{(b)} = \sum_{z\in\DB}(I_{t,z}^{(b)}-q)h_{t,z}^{(b)} +b(B_t^{(b)}-q)a_t^{(b)} +\sigma C G_t^{(b)}, \quad
  G_t^{(b)}\mid\cF_{t-1}^{(b)}\sim\multiNormal{0}{1}
\end{align*}
Consequently,
\begin{align*}
    \Ex{\xi_t^{(b)}\mid\cF_{t-1}^{(b)}} =0, \quad
    \Ex{(\xi_t^{(b)})^2\mid\cF_{t-1}^{(b)}} =q(1-q)\!\left[\sum_{z\in\DB}(h_{t,z}^{(b)})^2+b(a_t^{(b)})^2 \right]+\sigma^2C^2.
\end{align*}
and $(\xi_t^{(b)},\cF_t^{(b)})_{t\le T}$ is a martingale-difference sequence.
\end{proposition}

The canary score aggregates the one-step canary observations in \Cref{prop:exact-adaptive-decomposition}. The following definition introduces the relevant notation for its deterministic mean $\mu_{b,T}$, centered fluctuation $D_{b,T}$, martingale sum $M_{b,T}$, unconditional martingale variance $v_{b,T}$, and predictable quadratic variation $Q_{b,T}$.

\begin{definition}[Score decomposition]
\label{def:adaptive-score-decomposition}
Fix a world $b\in\bits$ and a number of training steps $T$. Write $X_{T,t}^{(b)}=m_{T,t}^{(b)}+\xi_{T,t}^{(b)}$ by \Cref{prop:exact-adaptive-decomposition}, and define
\begin{align*}
    S_T^{(b)} &\defin \frac{1}{\sqrt T}\sum_{t=1}^T X_{T,t}^{(b)}, &\mu_{b,T} &\defin \frac{1}{\sqrt T}\sum_{t=1}^T \Ex{m_{T,t}^{(b)}}\\
    D_{b,T} &\defin \frac{1}{\sqrt T}\sum_{t=1}^T \left(m_{T,t}^{(b)}-\Ex{m_{T,t}^{(b)}}\right), & M_{b,T} &\defin \frac{1}{\sqrt T}\sum_{t=1}^T\xi_{T,t}^{(b)} \\
    v_{b,T} &\defin \frac{1}{T}\sum_{t=1}^T \Ex{(\xi_{T,t}^{(b)})^2}, &Q_{b,T} &\defin \frac{1}{T}\sum_{t=1}^T \Ex{(\xi_{T,t}^{(b)})^2 \mid\cF_{T,t-1}^{(b)}}.
\end{align*}
The score admits the decomposition
\begin{align*}
    S_T^{(b)} =\mu_{b,T}+D_{b,T}+M_{b,T},
\end{align*}
\end{definition}

\Cref{def:adaptive-score-decomposition} gives the score representation with the full adaptive background trajectory. To distinguish the DP-SGD noise multiplier from the standard deviation of its one-dimensional projection, write $\tau \defin \sigma C>0.$

\paragraph{From the adaptive decomposition to Model~\ref{model:canary score}.}
The preceding decomposition allows the background projections, canary
magnitude, and canary direction to depend on the training history.  The next
two assumptions separate the roles of background suppression and canary
saturation.  A fixed-coordinate Dirac gradient canary enforces exact saturation
by construction, but it does not by itself suppress the projections of
ordinary clipped gradients; background suppression is a separate condition.
Once the background is suppressed, a predictable change in direction does not
alter the scalar noise law: isotropy implies that the projection of each fresh
Gaussian noise vector onto any past-determined unit direction has law
$\multiNormal{0}{\tau^2}$ and is independent of the past.  Exact saturation
therefore gives a finite-$T$ reduction to Model~\ref{model:canary score}.

\begin{assumption}[Background suppression]
\label{assump:background-suppression}
For every $t\in[T]$, $z\in\DB$, and $b\in\bits$,
\begin{align*}
    h_{t,z}^{(b)}=0 \qquad \text{almost surely}.
\end{align*}
\end{assumption}

\begin{lemma}[Exact canary saturation]
\label{assump:exact-canary-saturation}
For every $t\in[T]$, the present-world canary magnitude satisfies
\begin{align*}
    a_t^{(1)}=C \qquad 
\end{align*}
\end{lemma}

\begin{proposition}[Fixed-magnitude reduction; the proof is in~\Cref{app:proof-fixed-magnitude-reduction}]
\label{prop:fixed-magnitude-reduces-model-one}
Under
\Cref{assump:background-suppression,assump:exact-canary-saturation}, for
each fixed $b\in\bits$, the sequence
$\bigl(B_t^{(b)},G_t^{(b)}\bigr)_{t=1}^T$ is i.i.d. with common law
$\ber{q}\otimes\multiNormal{0}{1}$, and
\begin{align*}
    X_t^{(b)}=bB_t^{(b)}C+\tau G_t^{(b)}.
\end{align*}
Consequently, for the scores defined in
\Cref{def:adaptive-score-decomposition},
\begin{align*}
    S_T^{(0)}\sim\multiNormal{0}{\tau^2},\\
    S_T^{(1)}=\frac{C}{\sqrt T}\sum_{t=1}^T B_t^{(1)}+\frac{\tau}{\sqrt T}\sum_{t=1}^T G_t^{(1)}, \quad \Ex{S_T^{(1)}}&=\sqrt T\,qC, \quad
    \Var{S_T^{(1)}}&=\tau^2+q(1-q)C^2.
\end{align*}
\end{proposition}

%% file: 106-app-missing-proofs-dependent-dpsgd.tex
\section{Missing Proofs for Appendix~\ref{app:dependent-dpsgd}}
\label{app:missing-proofs-dependent-dpsgd}

\begin{proof}[Proof of Proposition~\ref{prop:exact-adaptive-decomposition}]
\label{app:proof-exact-adaptive-decomposition}
Fix $t\in[T]$ and $b\in\bits$. By
\Cref{alg:dpsgd-transcript,def:canary-aligned-obs},
\begin{align*}
    X_t^{(b)} ={} qn\innerProd{u_t^{(b)}}{\widetilde g_t^{(b)}} 
    ={}\sum_{z\in\DB}I_{t,z}^{(b)}h_{t,z}^{(b)} +bB_t^{(b)}a_t^{(b)} +\innerProd{u_t^{(b)}}{\rZ_t^{(b)}} .
\end{align*}
The direction $u_t^{(b)}$ is $\cF_{t-1}^{(b)}$-measurable and has unit norm, with $\rZ_t^{(b)} \mid\cF_{t-1}^{(b)} \sim\multiNormal{0}{\sigma^2C^2I_d}.$ Rotational invariance therefore gives a random variable $G_t^{(b)}$ such that
\begin{align*}
    \innerProd{u_t^{(b)}}{\rZ_t^{(b)}} =\sigma C G_t^{(b)}, \qquad
    G_t^{(b)} \mid\cF_{t-1}^{(b)} \sim\multiNormal{0}{1}.
\end{align*}
Adding and subtracting the conditional means of the sampling indicators, we have
\begin{align*}
    X_t^{(b)}={}& q\!\left[\sum_{z\in\DB}h_{t,z}^{(b)}+b\,a_t^{(b)}\right] + 
    \sum_{z\in\DB}(I_{t,z}^{(b)}-q)h_{t,z}^{(b)} +b(B_t^{(b)}-q)a_t^{(b)} +\sigma C G_t^{(b)}\\
    ={}&m_t^{(b)}+\xi_t^{(b)}.
\end{align*}
Hence $m_t^{(b)}$ is the predictable conditional mean of $X_t^{(b)}$, whereas $\xi_t^{(b)}=X_t^{(b)}-m_t^{(b)}$ is its centered perturbation. 

Conditionally on $\cF_{t-1}^{(b)}$, the variables $I_{t,z}^{(b)}\sim\ber{q}$, $B_t^{(b)}\sim\ber{q}$, and $G_t^{(b)}\sim\multiNormal{0}{1}$ are mutually independent, while the coefficients $h_{t,z}^{(b)}$ and $a_t^{(b)}$ are measurable. Thus the conditional expectations below average only over the fresh step-$t$ sampling indicators and Gaussian noise, with the past fixed. It follows that $\Ex{\xi_t^{(b)}\cond\cF_{t-1}^{(b)}}=0.$ Also, we have
\begin{align*}
    \Ex{(\xi_t^{(b)})^2\cond\cF_{t-1}^{(b)}}
    ={}& q(1-q)\sum_{z\in\DB}(h_{t,z}^{(b)})^2 +b^2q(1-q)(a_t^{(b)})^2 +\sigma^2C^2\\
    ={}& q(1-q)\!\left[\sum_{z\in\DB}(h_{t,z}^{(b)})^2 +b(a_t^{(b)})^2 \right] +\sigma^2C^2,
\end{align*}
where the final equality uses $b^2=b$ for $b\in\bits$.

Finally, since $\xi_t^{(b)}$ is $\cF_t^{(b)}$-measurable and square-integrable: $\abs{h_{t,z}^{(b)}}\le C$, $0\le a_t^{(b)}\le C$, and the Gaussian term has finite second moment. Together with its zero conditional mean, this proves that $(\xi_t^{(b)},\cF_t^{(b)})_{t\le T}$ is a martingale-difference sequence.
\end{proof}

\begin{proof}[Proof of~\Cref{prop:fixed-magnitude-reduces-model-one}]
\label{app:proof-fixed-magnitude-reduction}
Under Assumption \ref{assump:background-suppression} and Lemma \ref{assump:exact-canary-saturation}, the general formulas in \Cref{prop:exact-adaptive-decomposition} reduce to
\begin{align*}
    m_t^{(b)}=bqC, \qquad \xi_t^{(b)}&=b(B_t^{(b)}-q)C+\tau G_t^{(b)}.
\end{align*}
Here the factor $b$ makes the value of $a_t^{(0)}$ irrelevant in the absent world.  Adding these two terms gives
\begin{align*}
    X_t^{(b)}=bB_t^{(b)}C+\tau G_t^{(b)}.
\end{align*}

It remains to justify that the history-dependent direction does not create
scalar temporal dependence.  Conditional on $\cF_{t-1}^{(b)}$, the vector
$u_t^{(b)}$ is fixed and has unit norm, whereas the fresh Gaussian noise is
isotropic with law $\multiNormal{0}{\tau^2I_d}$.  Therefore
$G_t^{(b)}\mid\cF_{t-1}^{(b)}\sim\multiNormal{0}{1}$.  To see that
$\bigl(B_t^{(b)},G_t^{(b)}\bigr)_{t=1}^T$ is i.i.d. with common law
$\ber{q}\otimes\multiNormal{0}{1}$, examine the conditional characteristic
function of the pair at time $t$.  For every $r,s\in\Real$,
\begin{align*}
    \Ex{\rExp{irB_t^{(b)}+isG_t^{(b)}} \mid\cF_{t-1}^{(b)}} =\bigl[(1-q)+qe^{ir}\bigr]e^{-s^2/2}.
\end{align*}
The right-hand side is deterministic, does not depend on
$\cF_{t-1}^{(b)}$, and is the characteristic function of
$\ber{q}\otimes\multiNormal{0}{1}$.  Iterating this identity with the tower
property proves the joint characteristic-function factorization across time,
and hence the asserted temporal independence.

Summing the one-step representation gives
\begin{align*}
    S_T^{(b)} =\frac{bC}{\sqrt T}\sum_{t=1}^T B_t^{(b)} +\frac{\tau}{\sqrt T}\sum_{t=1}^T G_t^{(b)}.
\end{align*}
The absent-world formula is exactly $\multiNormal{0}{\tau^2}$.  In the present world, independence and the Bernoulli and Gaussian moments give
\begin{align*}
    \Ex{S_T^{(1)}}=\sqrt T\,qC, \qquad \Var{S_T^{(1)}}=q(1-q)C^2+\tau^2.
\end{align*}
Finally, $m_t^{(b)}=bqC$ is deterministic, so $D_{b,T}=0$, and the conditional second-moment formula in \Cref{prop:exact-adaptive-decomposition} becomes the deterministic value $\tau^2+bq(1-q)C^2$.  Its time average and expectation are therefore both the same value, which proves $Q_{b,T}=v_{b,T}=\tau^2+bq(1-q)C^2$ and completes the proof.
\end{proof}

%% file: 103-app-convergence.tex
\section{Convergence theory for sum of Gaussian mixture}\label{app:converge}

\begin{definition}[Two-component centered Gaussian mixture]
\label{def:two-comp-mixture}
A real random variable $\rX$ is a \emph{centered two-component Gaussian mixture} with parameters $(w,v,\mu,\rho,\sigma)$, written $\rX\sim\GM{w}{v}{\mu}{\rho}{\sigma}$, if its density function is
\begin{align*}
    p_{\rX}(x)
    \;=\;
    \frac{w}{\sqrt{2\pi}\,\sigma}\rExp{-(x-\mu)^2/(2\sigma^2)}
    \;+\;
    \frac{v}{\sqrt{2\pi}\,\sigma}\rExp{-(x-\rho)^2/(2\sigma^2)},
\end{align*}
where $w,v\ge 0$, $w+v=1$, $\sigma>0$, and the centering condition $w\mu + v\rho = 0$ holds, so that $\Ex{\rX}=0$.
\end{definition}

For a real random variable $\rY$, the moment generating function $\MGF{\rY}(t)\defin\Ex{\rExp{t\rY}}$ and the cumulant generating function $\CGF{\rY}(t)\defin\log \MGF{\rY}(t)$. Further define $\ell$-th cumulant of $\rY$ as
\begin{align*}
    \cum{\ell}{\rY} \defin \lim_{t\to 0}\frac{d^\ell}{dt^\ell} \CGF{\rY}(t).
\end{align*}

For a Gaussian distribution $G\sim\multiNormal{0}{\tau^2}$, $\CGF{G}(t) = \tau^2 t^2/2$, and so $\cum{2}{G} = \tau^2$ and $\cum{\ell}{G} = 0$ for all $\ell\ge 3$. 

Because the MGF of a two-component Gaussian mixture is finite in a neighborhood of $t=0$, the moment problem for $\rY$ is determinate: the cumulant sequence $\{ \cum{l}\rY \}_{l \geq 1}$ uniquely determines the law of $Y$. Hence, if a sequence of mixture-derived random variables has cumulants of order $\geq 3$ shrinking to zero, the limit law is necessarily Gaussian, and the rate of cumulant decay quantifies the rate of convergence.

\begin{theorem}[Cumulant expansion of normalized mixture sums]
\label{thm:cumulant-expansion}
Let $\rX_1,\ldots,\rX_n$ be independent with $\rX_i\sim\GM{w_i}{v_i}{\mu_i}{\rho_i}{\sigma_i}$ for each $i\in[n]$, and define
\begin{align*}
    \rS_n \defin \frac{1}{\sqrt{n}}\rSum{i}{1}{n} \rX_i.
\end{align*}
Then $\cum{0}{\rS_n}=\cum{1}{\rS_n}=0$, and
\begin{align*}
    \cum{2}{\rS_n}
    &= \frac{1}{n}\rSum{i}{1}{n}\sigma_i^2
    +\frac{1}{n}\rSum{i}{1}{n}\bigl(w_i\mu_i^2 + v_i\rho_i^2\bigr),
    \\
    \cum{3}{\rS_n}
    &= \frac{1}{n^{3/2}}\rSum{i}{1}{n}\bigl(w_i\mu_i^3 + v_i\rho_i^3\bigr),
    \\
    \cum{4}{\rS_n} &= \frac{1}{n^{2}}\rSum{i}{1}{n}\Bigl[w_i\mu_i^4 + v_i\rho_i^4 -3\bigl(w_i\mu_i^2 + v_i\rho_i^2\bigr)^2\Bigr],
\end{align*}
and more generally, for every $\ell\ge 3$,
\begin{align*}
    \cum{\ell}{\rS_n} = \frac{\alpha_\ell}{n^{\ell/2-1}},
    \qquad
    \alpha_\ell \defin \frac{1}{n}\rSum{i}{1}{n}\beta_{\ell,i},
\end{align*}
where $\beta_{\ell,i}$ is a polynomial of degree $\ell$ in $(\mu_i,\rho_i)$ with coefficients depending only on $(w_i,v_i)$. 
\end{theorem}

\begin{proof}

For $\rX_i\sim\GM{w_i}{v_i}{\mu_i}{\rho_i}{\sigma_i}$, conditioning on the mixture component and using $\Ex{\rExp{t\rN}} = \rExp{m t + s^2 t^2/2}$ for $\rN\sim\multiNormal{m}{s^2}$ gives
\begin{align*}
    \MGF{\rX_i}(t)
    &= \rExp{\sigma_i^2 t^2/2}\Bigl(w_i\rExp{\mu_i t} + v_i\rExp{\rho_i t}\Bigr) \\
    \implies \MGF{\rS_n}(t)
    &= \rProd{i}{1}{n} \MGF{\rX_i}\!\left(\frac{t}{\sqrt n}\right)
    = \rExp{\frac{t^2}{2n}\rSum{i}{1}{n}\sigma_i^2}\,
    \rProd{i}{1}{n}\Bigl(w_i\rExp{\mu_i t/\sqrt n} + v_i\rExp{\rho_i t/\sqrt n}\Bigr). \tag{by independence}
\end{align*}

Taking logarithms,
\begin{align*}
    \CGF{\rS_n}(t) &= \frac{t^2}{2n}\rSum{i}{1}{n}\sigma_i^2+\rSum{i}{1}{n}\log\!\Bigl(w_i\rExp{\mu_i t/\sqrt n} + v_i\rExp{\rho_i t/\sqrt n}\Bigr). \\
\end{align*}

Expanding the exponentials as power series:
\[
     w_i\rExp{\mu_i t/\sqrt n} + v_i\rExp{\rho_i t/\sqrt n} =
    \rSum{k}{0}{\infty}\frac{t^k}{k!\, n^{k/2}}\bigl(w_i\mu_i^k + v_i\rho_i^k\bigr).
\]

The $k=0$ term is $w_i+v_i=1$, and the $k=1$ term is $\frac{t}{\sqrt n}(w_i\mu_i+v_i\rho_i)=0$ by the centering condition. Hence we can rewrite:
\[
    w_i\rExp{\mu_i t/\sqrt n} + v_i\rExp{\rho_i t/\sqrt n}= 1 + f_i(t),
    \qquad \text{where }
    f_i(t) \defin \rSum{k}{2}{\infty}\frac{t^k}{k!\,n^{k/2}}\bigl(w_i\mu_i^k+v_i\rho_i^k\bigr).
\]
Applying Taylor series for log we get:
\[
    \CGF{\rS_n}(t)= \frac{t^2}{2n}\rSum{i}{1}{n}\sigma_i^2+ \rSum{i}{1}{n}\rSum{\ell}{1}{\infty}\frac{(-1)^{\ell+1}}{\ell}\bigl[f_i(t)\bigr]^\ell.
\]

Given the above equation, we can compute $m!$ times the coefficient of $t^m$ to extract the cumulants of $\rS_n$.
\begin{itemize}
    \item Case $m=2$: Two contributions: the explicit Gaussian-noise term $\frac{t^2}{2n}\sum_i\sigma_i^2$, and the $\ell=1$, $k=2$ term of $f_i(t)$, which is $\frac{t^2}{2n}\sum_i(w_i\mu_i^2+v_i\rho_i^2)$. The $\ell\ge 2$ contributions vanish at order $t^2$, since $f_i(t)^\ell$ already starts at $t^{2\ell}$. 
    \item Case $m=3$: For $\ell\ge 2$, $f_i(t)^\ell$ starts at $t^{2\ell}\ge t^4$, contributing nothing at order $t^3$. Hence only $\ell=1$, $k=3$ contributes:
\begin{align*}
    \rSum{i}{1}{n}\frac{1}{3!\,n^{3/2}}(w_i\mu_i^3+v_i\rho_i^3).
\end{align*}
Multiplying by $3!$ gives the exact form.
    \item Case (general $m \ge 3$): The lowest-order term of $f_i(t)^\ell$ is $t^{2\ell}/n^\ell$, so only finitely many values of $\ell$ contribute to $t^m$, namely $\ell\in\{1,2,\ldots,\lfloor m/2\rfloor\}$. Each such contribution to $t^m$ carries a factor $1/n^{m/2}$ (from the $f_i(t)^\ell$ expansion at order $t^m$, where the powers of $t$ and $1/\sqrt n$ track together). The per-$i$ coefficient is therefore of the form $\beta_{\ell,i}/n^{m/2}$ for a polynomial $\beta_{\ell,i}$ in $(\mu_i,\rho_i)$, and summing over $i\in[n]$ produces a factor of $n$. After multiplication by $m!$, the cumulant takes the form $\text{(some bounded scalar)}/n^{m/2-1}$. This gives the general formula for the cumulants of a Gaussian mixture.
\end{itemize}
\end{proof}

Applying the above theorem to the case where $X_i$ are i.i.d. samples gives us the following asymptotic rate of convergence for all cumulants:

\begin{corollary}[Rate of convergence]
\label{cor:rate}
In \Cref{thm:cumulant-expansion} if the summands are i.i.d., i.e., $\rX_1,\ldots,\rX_n\iid\GM{w}{v}{\mu}{\rho}{\sigma}$ for some fixed parameters $(w,v,\mu,\rho,\sigma)$. Then $\cum{\ell}{\rS_n} = O(n^{-(\ell/2-1)})$ for every $\ell\ge 3$, with the leading non-vanishing cumulant of order $\ge 3$ determined as follows.
\begin{itemize}
    \item If $w\mu^3 + v\rho^3 \neq 0$, then
        \begin{align*}
            \cum{3}{\rS_n} = \Theta(n^{-1/2}),
        \end{align*}
    and the third cumulant is the leading non-vanishing one.
    
    \item If $w\mu^3 + v\rho^3 = 0$, then $\cum{3}{\rS_n} = 0$ and
        \begin{align*}
            \cum{4}{\rS_n} = \Theta(n^{-1}),
        \end{align*}
    and the fourth (or higher) cumulant is the leading non-vanishing one.
\end{itemize}
\end{corollary}

\paragraph{Quantitative Kolmogorov-distance estimate}

The cumulant decay above is not just an asymptotic statement: it directly
controls the Kolmogorov distance between the true law $\trueLaw$ of the standardized audit statistic and its Gaussian model $\modelLaw$ via the classical Edgeworth expansion:

\begin{lemma}[Edgeworth expansion, {\cite{petrov2000classical,petrov2012sums}}]
\label{lem:edgeworth}
Let $\rY_1,\ldots,\rY_n$ be i.i.d. random variables with mean $0$, variance $\tau^2 > 0$, and distribution satisfying Cram\'er's condition
$\limsup_{|t|\to\infty}|\Ex{e^{it\rY_1}}| < 1$. Let $F_n$ be the CDF of
$(\tau\sqrt n)^{-1}\rSum{i}{1}{n}\rY_i$. Then
\begin{align}
\label{eq:edgeworth}
F_n(x)-\Phi(x)=
-\varphi(x)\!\left[
    \frac{\gamma_1}{6\sqrt n}\,h_2(x)
    + \frac{1}{n}\!\left(
        \frac{\gamma_2}{24}\,h_3(x)
        + \frac{\gamma_1^{\,2}}{72}\,h_5(x)
    \right)
\right]
+ O\!\bigl(n^{-3/2}\bigr),
\end{align}
where $\gamma_1 \defin \cum{3}{\rY_1}/\tau^3$,
$\gamma_2\defin \cum{4}{\rY_1}/\tau^4$, $\varphi$ is the standard Gaussian
density, and $h_2(x)=x^2-1$, $h_3(x)=x^3-3x$, $h_5(x)=x^5-10x^3+15x$ are
probabilists' Hermite polynomials.
\end{lemma}

The expansion does not require any computation beyond the cumulants of
$\auditStat$ already derived: setting $\rY_t = \widetilde\rX_t$ and
$n = T$, the Edgeworth coefficients equal the standardized cumulants of
$\auditStat$ itself,
\begin{align*}
    \frac{\gamma_1}{\sqrt T}
    = \frac{\cum{3}{\auditStat}}{\cum{2}{\auditStat}^{3/2}},
    \qquad
    \frac{\gamma_2}{T}
    = \frac{\cum{4}{\auditStat}}{\cum{2}{\auditStat}^{2}}.
\end{align*}

Using the above general statement we can derive the closed-form Kolmogorov-distance bound for the DP audit statistic:

\begin{proof}[Proof of \Cref{thm:dp-kolmogorov-bound}]
The summands in normalized distribution
$\auditStat = T^{-1/2}\rSum{t}{1}{T}\widetilde\rX_t$ satisfy
$$\widetilde\rX_t\iid\GM{1-q}{q}{-qC}{(1-q)C}{\sigma},$$ an i.i.d.\ centered
Gaussian mixture. Its distribution has a smooth density, so Cram\'er's
condition holds and~\Cref{lem:edgeworth} applies with
$\rY_t = \widetilde\rX_t$, $\tau^2 = \cum{2}{\widetilde\rX_t} = \sigma^2 + q(1-q)C^2$, and $n = T$. Using the cumulant formulas already derived,
\begin{align*}
    \gamma_1
    =  \frac{q(1-q)(1-2q)\,C^3}{\bigl[\sigma^2 + q(1-q)C^2\bigr]^{3/2}}.
\end{align*}
The Kolmogorov distance is invariant under common affine rescaling of
$\trueLaw$ and $\modelLaw$, so it equals the standardized
$\sup_{x\in\Real}|F_T(x)-\Phi(x)|$ from~\eqref{eq:edgeworth}. The function
$\varphi(x)(x^2-1)$ attains its maximal absolute value at $x=0$ with
$|\varphi(0)\cdot(-1)|=1/\sqrt{2\pi}$, and the next-order Edgeworth
contributions are uniformly $O(1/T)$. Hence
\begin{align*}
    \dK{\trueLaw}{\modelLaw}
    =
    \frac{|\gamma_1|}{6\sqrt T}\cdot\frac{1}{\sqrt{2\pi}}+ O(1/T)
\end{align*}

\end{proof}

\paragraph{Improved concrete estimate for convergence rate}
For the DP-SGD auditing parameters, we can in fact get tighter bound than the above global Kolmogorov-distance estimate. The intuition is that the audit only ever queries the CDF in one specific region, and hence we need the Gaussian approximation to be accurate only where the audit looks.

Let $F_T$ be the CDF of the standardized audit statistic $\auditStat$; and let the functions $\varphi$ and $\Phi$ respectively represent the standard Gaussian density function and its corresponding CDF. Then we have the following tighter bound: 
\begin{corollary}[Tail-localized CDF deviation]
\label{cor:tail-deviation}
Under the setting of~\Cref{thm:dp-kolmogorov-bound}, for every threshold
$x_0 \ge \sqrt 3$,
\begin{align*}
    \sup_{|x|\ge x_0}|F_T(x)-\Phi(x)|=
    \frac{q(1-q)\,|1-2q|\,C^3}
         {6\sqrt T\,\bigl[\sigma^2 + q(1-q)C^2\bigr]^{3/2}}
    \cdot \varphi(x_0)\,(x_0^2-1)+ O(1/T).
\end{align*}
\end{corollary}
\begin{proof}[Proof of \cref{cor:tail-deviation}]
The Edgeworth expansion \Cref{lem:edgeworth} gives the pointwise expansion
$F_T(x)-\Phi(x) = -\varphi(x)(x^2-1)\,\gamma_1/(6\sqrt T) + O(1/T)$
uniformly in $x$. By the monotonicity of $\varphi(x)(x^2-1)$ on
$|x|\ge\sqrt 3$, the supremum on $\{|x|\ge x_0\}$ is attained at $|x|=x_0$,
giving the stated identity after substituting
$\gamma_1=q(1-q)(1-2q)C^3/[\sigma^2+q(1-q)C^2]^{3/2}$.
\end{proof}

For concrete parameters used in our experiments: $T=2500$, $\epsilon=8$, $\sigma=2.6245$, $q=0.0819$, and $C=1$, the prefactor $ \frac{q(1-q)\,|1-2q|\,C^3}{6\sqrt T\,(\sigma^2 + q(1-q)C^2)^{3/2}} \approx 1.14\times 10^{-5}$, and the tail factor $\varphi(x_0)(x_0^2-1)$ controls how small the deviation becomes as follows:
\begin{center}
\begin{tabular}{c|c|c}
$x_0$ & $\varphi(x_0)(x_0^2-1)$ & $\sup_{|x|\ge x_0}|F_T(x)-\Phi(x)|$ \\\hline
$3$   & $3.55\times 10^{-2}$    & $4.04\times 10^{-7}$ \\
$4$   & $2.01\times 10^{-3}$    & $2.29\times 10^{-8}$ \\
$5$   & $3.58\times 10^{-5}$    & $4.08\times 10^{-10}$\\
\end{tabular}
\end{center}

The threshold $x_0=4$ corresponds to a Gaussian tail mass
$\Phi(-4)\approx 3\times 10^{-5}$, which is the regime relevant for our $\delta = 10^{-5}$ audit. Within this audit-relevant tail the Gaussian model matches the true distribution to within $\sim 10^{-8}$, which is three orders of magnitude smaller. This gives more confidence that the empirical privacy lower bound produced by the auditor closely tracks the theoretical guarantee in our experimental setup.

%% file: 107-experimental-details.tex
\section{Experimental details for DP-FTRL and DP-MF-FTRL}
\label{app:ftrl}
For DP-FTRL, $\mathbf{C}_{\nu,t}=1$ iff leaf $t$ lies below
node $\nu$, and each node is released with independent noise. Every leaf has exactly
$L=\lceil\log_2 T\rceil+1$ ancestors, so $\|\mathbf{C}_{:,t}\|_2=\sqrt{L}$
for all $t$, and the score is the sum of the released node values on the
ancestor path of $t_i$, divided by $\sqrt{L}$. The Gaussian pair is
\[
G_0=\mathcal{N}\big(0,\;\sigma_{\mathrm{node}}^2\big),
\qquad
G_1=\mathcal{N}\big(\sqrt{L}\,C,\;\sigma_{\mathrm{node}}^2\big),
\]
and every leaf is a worst-case placement.

Whereas in DP-FTRL, although the released updates exhibit correlated noise, the decoder $\mathbf{B}$ is invertible, so the auditor can recover $y$ from the decoded releases.

We calibrate the noise with the analytic Gaussian
mechanism \cite{ballewang}, at sensitivity $\sqrt{L}\,C$ for the tree
and at sensitivity $C\max_t\|\mathbf{C}_{:,t}\|_2$ for the factorization. As the
encoder we use the square root of the prefix-sum matrix, which is lower
triangular and Toeplitz, has a closed form, and needs no numerical optimization. Both mechanisms train the same ScatterLinear
classifier on CIFAR-10, a $J=2$ scattering transform
\cite{andreux2020kymatio} followed by GroupNorm and a linear layer, in a
single pass with $T=128$ rounds, logical batch size $B=380$, and $m=5000$ Dirac
canaries. At $T=128$ the factorization needs $1.75$ times less noise for the same privacy guarantee. The per-canary statistic of \cite{andrew2024oneshot} is the round-wise max cosine $\tilde s_i = \max_{t \in [0,T)} \langle e_{c_i},\, G_t\rangle / \lVert G_t\rVert_2$, where $G_t$ is the released cumulative iterate at round $t$. 

%% file: 104-app-ablations.tex
\section{Additional Experiments}\label{app:additional-ablations}
This section provides supplementary plots referenced in Section~\ref{sec:exp}, including adaptive clipping for DP-SGD, robustness to hyperparameters, canary-out score distribution, a hexbin density visualization of canary direction independence, and a comparison of confidence region constructions across varying canary counts.

\subsection{Adaptive clipping in DP-SGD}\label{app:adaptive-clipping-app}
In the adaptive clipping algorithm, the clipping bound at each step ($C_t$) is noised and released publicly \cite{adaptiveclipping}. Since the adaptive clipping threshold is released publicly, the auditor can normalize each canary score by the corresponding clipping bound. This reduces the score process to the same normalized model analyzed in Model 1. The nearly identical empirical results in the table below suggest that the background gradient channel remains unchanged.

\begin{table}[h]
\label{tab:adaptive-clipping}
\centering
\begin{tabular}{lcccc}
\toprule
 & Steinke et al. & Mahloujifar et al. & Our auditor & True epsilon \\
\midrule
Fixed clipping    & 3.330 & 4.697 & 6.651 & 8 \\
Adaptive clipping & 3.095 & 4.464 & 6.769 & 8 \\
\bottomrule
\end{tabular}
\caption{Extending our auditor to DP-SGD with adaptive clipping}
\end{table}

\subsection{Robustness to hyperparameters}
\textbf{Clipping norm: }Proposition 1 shows that when the DP noise is calibrated proportional to the clipping norm, changing the clipping norm scales both neighboring score distributions by the same factor. Because the hockey-stick divergence is invariant under a common invertible rescaling, the inferred privacy bound should remain essentially unchanged. Our experiment below confirms this prediction: increasing the clipping from 1 to 10 doesn't change our estimate too much, and it remains substantially tighter than both baselines.

\begin{table}[H]
\centering
\caption{Empirical lower bounds on $\varepsilon$ for different clipping bounds
$C$ (claimed $\varepsilon = 8$)}
\label{tab:clipping}
\footnotesize
\begin{tabular}{@{}lrrr@{}}
\toprule
Clipping bound $C$
  & Steinke et al. \cite{steinke2023privacy}
  & Mahloujifar et al. \cite{mahloujifar2025auditing}
  & Ours \\
\midrule
1  & 3.330 & 4.697 & 6.679 \\
10 & 3.486 & 5.115 & 6.678 \\
\bottomrule
\end{tabular}
\end{table}

\textbf{Batch size: }The batch size affects the sampling probability. A smaller batch therefore produces fewer canary contributions over training and makes the audit statistically harder. This expected trend appears for all methods. Nevertheless, our auditor remains the strongest across all tested batch sizes, as shown below.

\begin{table}[H]
\centering
\caption{Empirical lower bounds on $\varepsilon$ for different batch sizes (claimed $\varepsilon = 8$)}
\label{tab:batchsize}
\footnotesize
\begin{tabular}{@{}lrrr@{}}
\toprule
Batch size
  & Steinke et al. \cite{steinke2023privacy}
  & Mahloujifar et al. \cite{mahloujifar2025auditing}
  & Ours \\
\midrule
1024 & 2.410 & 3.236 & 4.518 \\
2048 & 3.427 & 4.877 & 6.108 \\
4096 & 3.330 & 4.697 & 6.679 \\
\bottomrule
\end{tabular}
\end{table}

\textbf{Learning rate: }The learning rate has no substantial effect on any of the three auditors' estimates. This invariance is exactly what Model 1 predicts. The close agreement between this prediction and the empirical results across a range of learning rates provides additional evidence that Model 1's idealization captures the relevant dependencies well in practice.

\begin{table}[H]
\centering
\caption{Empirical lower bounds on $\varepsilon$ for different learning
rates (claimed $\varepsilon = 8$)}
\label{tab:lr}
\footnotesize
\begin{tabular}{@{}lrrr@{}}
\toprule
Learning rate
  & Steinke et al. \cite{steinke2023privacy}
  & Mahloujifar et al. \cite{mahloujifar2025auditing}
  & Ours \\
\midrule
2  & 3.330 & 4.697 & 6.645 \\
4  & 3.330 & 4.697 & 6.679 \\
8  & 3.346 & 4.727 & 6.699 \\
16 & 3.324 & 4.687 & 6.674 \\
\bottomrule
\end{tabular}
\end{table}

\subsection{Gaussianity of Canary Scores}
Complementing Figure~\ref{fig:hist-in} in the main text, which visualizes the distribution of canary-present (in) scores, Figure~\ref{fig:gaussianity} shows the empirical distribution of canary-absent (out) scores against the analytical Gaussian parameters. 

\begin{figure}[htbp]
    \centering
    \begin{subfigure}[t]{0.49\textwidth}
        \centering
        \includegraphics[width=\linewidth]
            {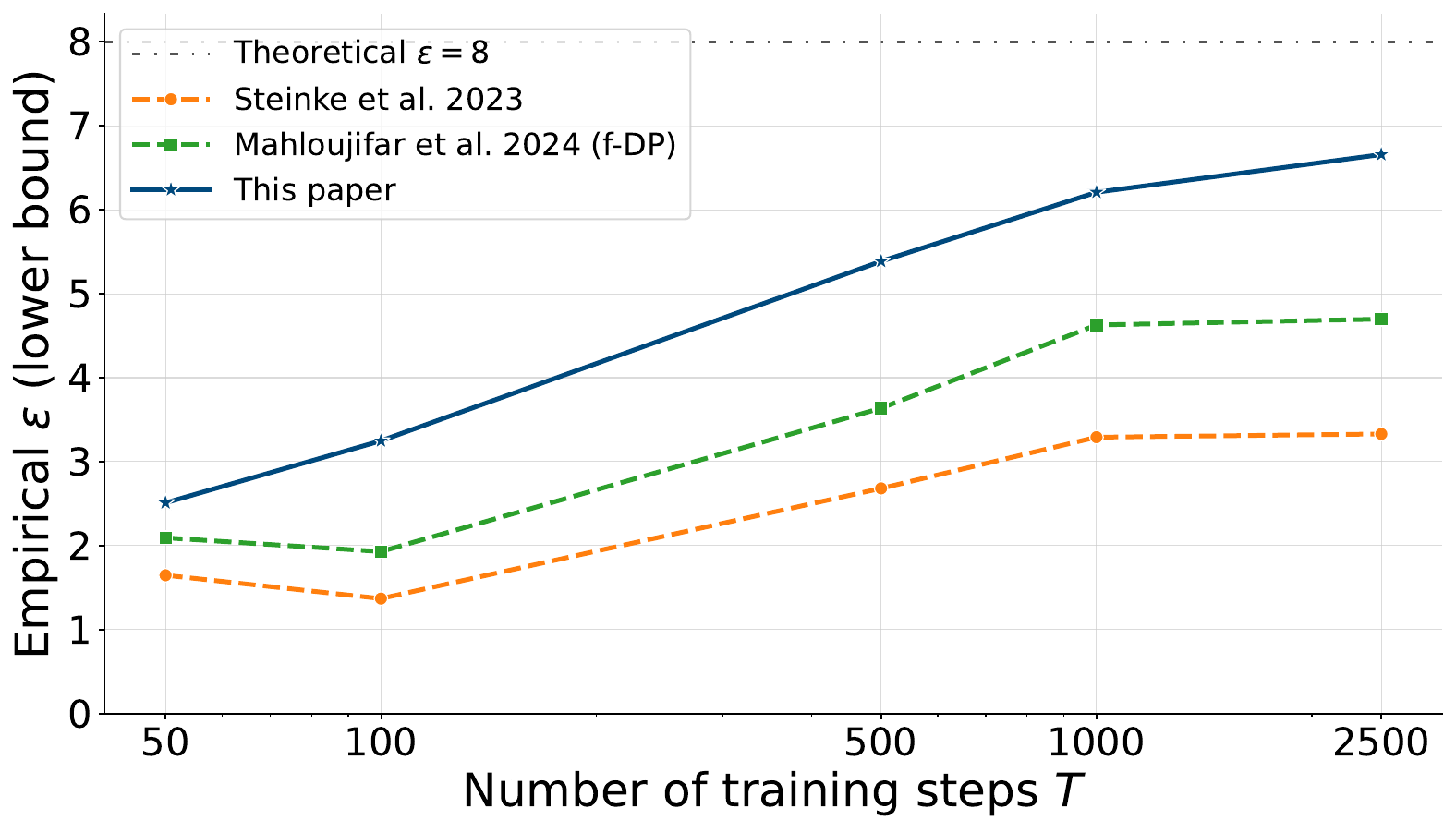}
        \caption{Auditing performance versus training steps.}
        \label{fig:ablation_T_eps8}
    \end{subfigure}
    \hfill
    \begin{subfigure}[t]{0.49\textwidth}
        \centering
        \includegraphics[width=\linewidth]
            {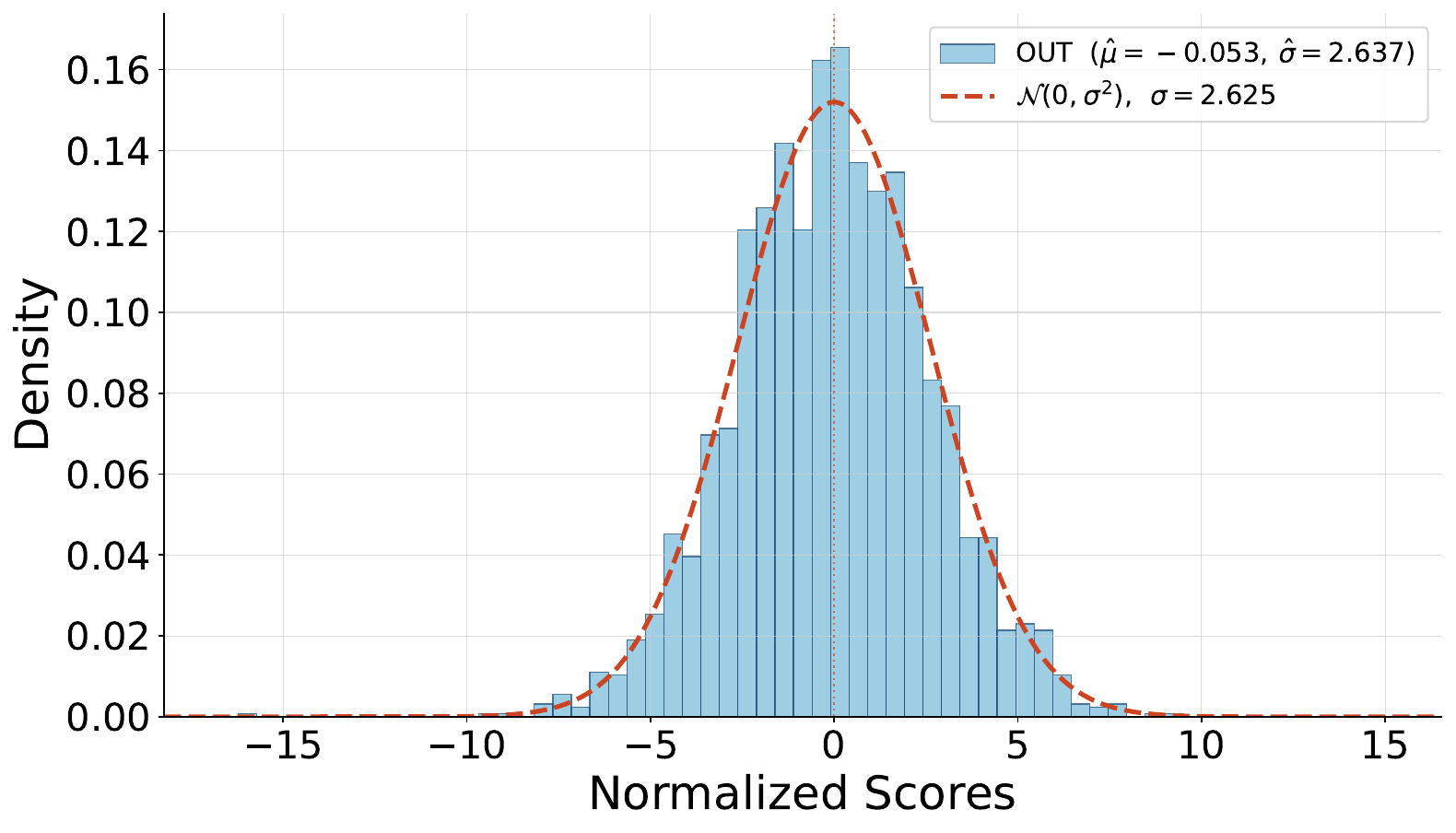}
        \caption{Canary-absent scores at $T=2500$.}
        \label{fig:gaussianity}
    \end{subfigure}
    \caption{DP-SGD ablations at theoretical $\epsilon=8$ and
    $m=5000$ canaries. In panel (b), the empirical score distribution
    closely matches the analytical Gaussian distribution
    $\mathcal{N}(0,\sigma^2)$.}
    \label{fig:dp_sgd_ablations}
\end{figure}

\subsection{Identical Distribution Across Canary Index}
Figure \ref{fig:hexbin} plots the standardized score $z_i = (s_i - \mu)/\sigma$ for each canary $i \in [1, m]$ as a hexbin density. If the marginal distribution of $s_i$ drifted with $i$ due to canary ordering correlating with training dynamics or embedding-space artifacts — we would expect to see vertical banding, trends, or shifts in the central mass. Instead, the density is uniform across the canary index axis, peaks symmetrically near zero, and remains contained within $[-3, 3]$, supporting the assumption that $s_i \sim \mathcal{N}(\mu, \sigma^2)$ uniformly across the canary population.

\begin{figure}[htbp]
    \centering
    \includegraphics[width=0.5\linewidth]{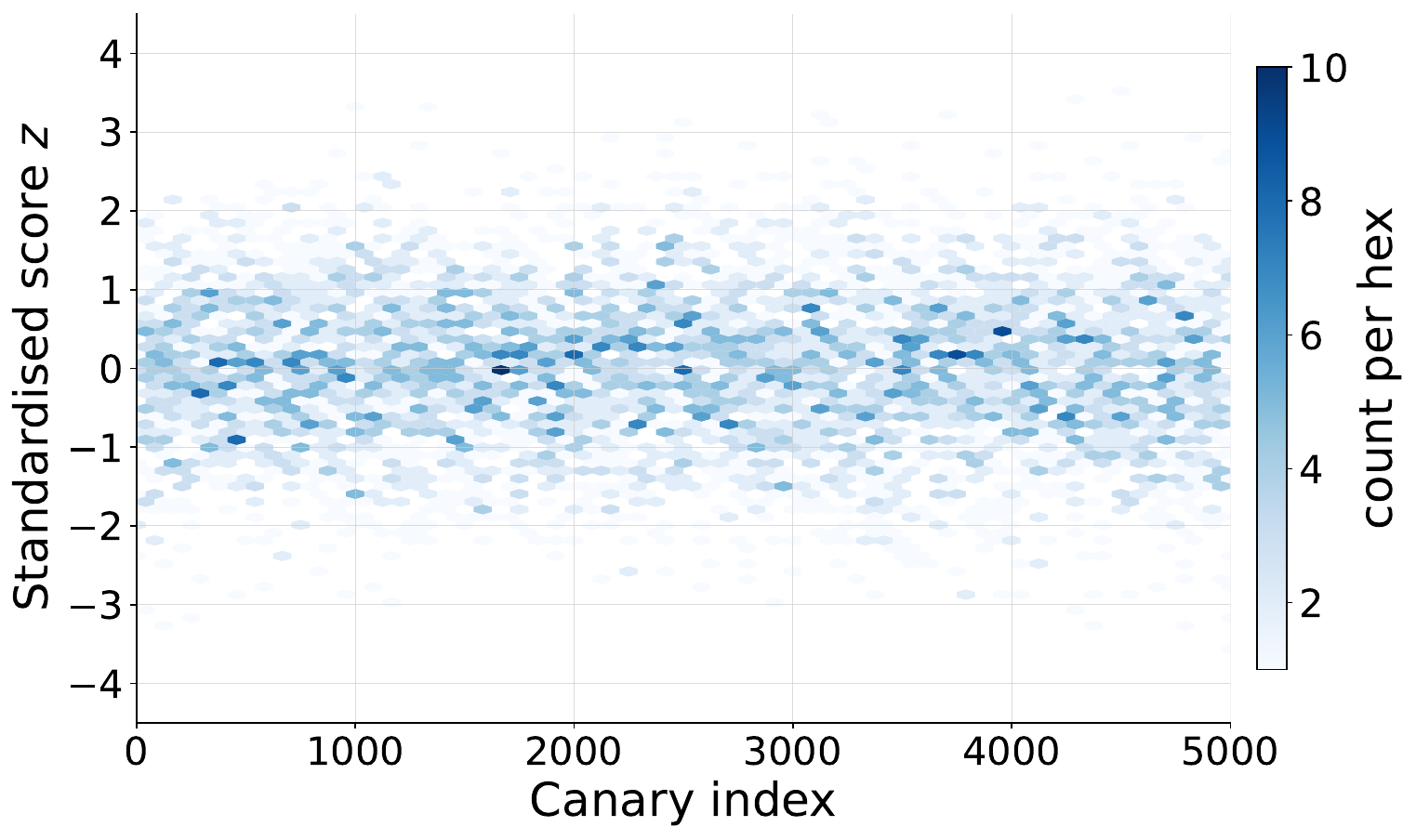}
    \caption{Hexbin density of standardized canary scores $z_i = (s_i - \mu)/\sigma$ plotted against canary index for DP-SGD at $\epsilon=8$, $m=5000$, $T=2500$.}
    \label{fig:hexbin}
\end{figure}

\subsection{Confidence Region Geometry vs.\ Sample Complexity}
As discussed in the \emph{Confidence Region Geometry} paragraph of Section~\ref{sec:exp}, we compare the Parametric Bonferroni and Bootstrap Ellipsoid confidence regions. Figure~\ref{fig:cr-sample-complexity} shows how the recovered empirical $\varepsilon$ varies with the number of canaries $m$ under each construction. The Ellipsoid consistently recovers higher $\varepsilon$ values across all sample sizes by capturing the cross-correlation between sample means and variances, which the axis-aligned Bonferroni rectangle ignores. The gap is especially pronounced at smaller $m$, where efficient use of statistical information matters most.

\begin{figure}[htbp]
    \centering
    \includegraphics[width=0.5\linewidth]{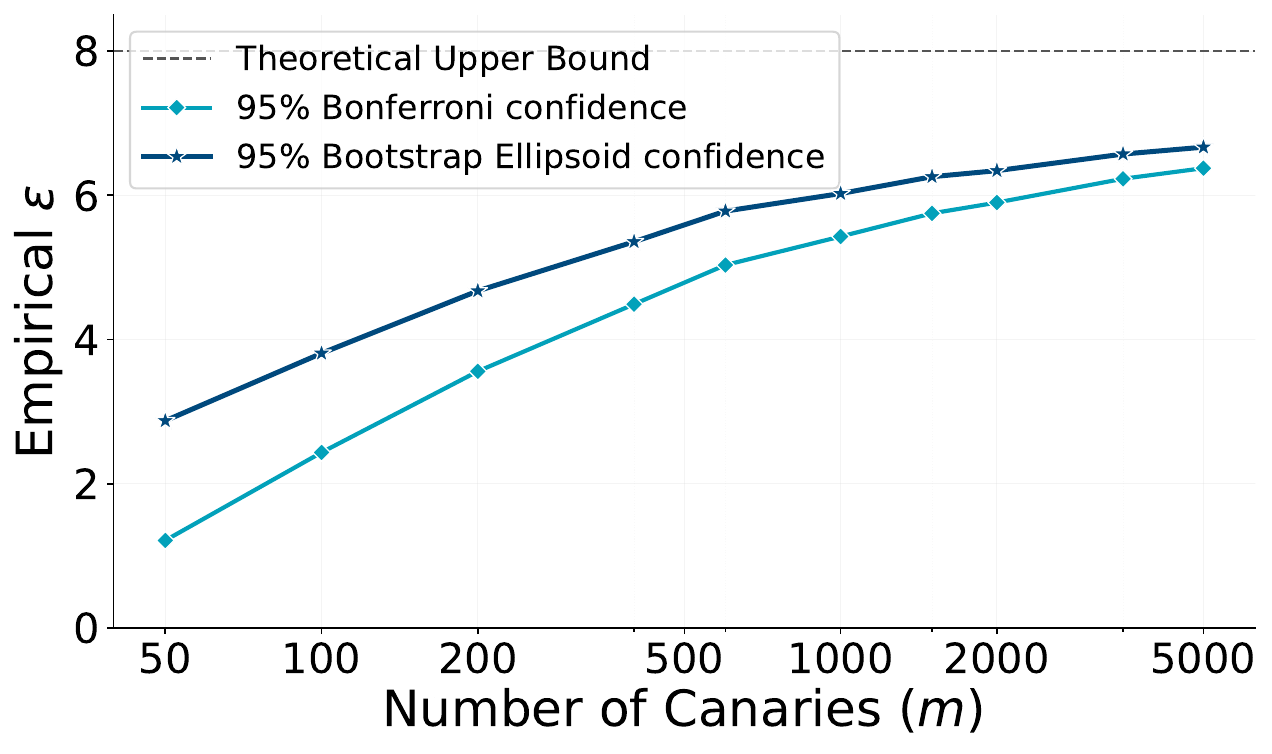}
    \caption{Comparison of empirical $\varepsilon$ recovered by the Parametric Bonferroni rectangle versus the Bootstrap Ellipsoid confidence region across varying canary counts $m$. The Ellipsoid yields tighter (higher) lower bounds at all sample sizes by exploiting the joint geometry of the mean-variance estimates.}
    \label{fig:cr-sample-complexity}
\end{figure}

%% file: references.bib
@article{dwork2014algorithmic,
  author  = {Cynthia Dwork and
             Aaron Roth},
  title   = {The Algorithmic Foundations of Differential Privacy},
  journal = {Foundations and Trends in Theoretical Computer Science},
  year    = {2014}
}

@article{jagielski2020auditing,
  title={Auditing differentially private machine learning: How private is private sgd?},
  author={Jagielski, Matthew and Ullman, Jonathan and Oprea, Alina},
  journal={Advances in Neural Information Processing Systems},
  volume={33},
  pages={22205--22216},
  year={2020}
}

@inproceedings{nasr2021adversary,
  title={Adversary instantiation: Lower bounds for differentially private machine learning},
  author={Nasr, Milad and Songi, Shuang and Thakurta, Abhradeep and Papernot, Nicolas and Carlin, Nicholas},
  booktitle={2021 IEEE Symposium on security and privacy (SP)},
  pages={866--882},
  year={2021},
  organization={IEEE}
}

@article{lu2023normal,
  author  = {Yun Lu and
             Malik Magdon{-}Ismail and
             Yu Wei and
             Vassilis Zikas},
  title   = {The Normal Distributions Indistinguishability Spectrum and Its Application to Privacy-Preserving Machine Learning},
  journal = {arXiv preprint arXiv:2309.01243},
  year    = {2023},
  note    = {Preliminary version to appear at the Theory and Practice of Differential Privacy (TPDP) 2026 workshop}
}

@article{LiMHMR15,
  author  = {Chao Li and
             Gerome Miklau and
             Michael Hay and
             Andrew McGregor and
             Vibhor Rastogi},
  title   = {The Matrix Mechanism: Optimizing Linear Counting Queries Under Differential Privacy},
  journal = {The {VLDB} Journal},
  volume  = {24},
  number  = {6},
  pages   = {757--781},
  year    = {2015},
  doi     = {10.1007/s00778-015-0391-1},
  url     = {https://doi.org/10.1007/s00778-015-0391-1}
}

@inproceedings{AbadiCGMMT016,
  author    = {Mart{\'{\i}}n Abadi and
               Andy Chu and
               Ian J. Goodfellow and
               H. Brendan McMahan and
               Ilya Mironov and
               Kunal Talwar and
               Li Zhang},
  title     = {Deep Learning with Differential Privacy},
  booktitle = {Proceedings of the 2016 {ACM} {SIGSAC} Conference on Computer and Communications Security},
  pages     = {308--318},
  year      = {2016},
  publisher = {{ACM}},
  doi       = {10.1145/2976749.2978318},
  url       = {https://doi.org/10.1145/2976749.2978318}
}

@inproceedings{lu2024eureka,
  author    = {Yun Lu and
               Malik Magdon{-}Ismail and
               Yu Wei and
               Vassilis Zikas},
  title     = {Eureka: A General Framework for Black-Box Differential Privacy Estimators},
  booktitle = {2024 IEEE Symposium on Security and Privacy (SP)},
  pages     = {913--931},
  year      = {2024},
  publisher = {IEEE}
}

@article{askin2025general,
  author        = {{\"{O}}nder Askin and
                   Holger Dette and
                   Martin Dunsche and
                   Tim Kutta and
                   Yun Lu and
                   Yu Wei and
                   Vassilis Zikas},
  title         = {General-Purpose {$f$}-{DP} Estimation and Auditing in a Black-Box Setting},
  journal       = {CoRR},
  volume        = {abs/2502.07066},
  year          = {2025},
  doi           = {10.48550/arXiv.2502.07066},
  eprint        = {2502.07066},
  archivePrefix = {arXiv},
  url           = {https://arxiv.org/abs/2502.07066}
}

@inproceedings{steinke2023privacy,
  author    = {Thomas Steinke and
               Milad Nasr and
               Matthew Jagielski},
  title     = {Privacy Auditing with One ({1}) Training Run},
  booktitle = {Advances in Neural Information Processing Systems},
  volume    = {36},
  year      = {2023},
  url       = {https://papers.nips.cc/paper_files/paper/2023/hash/9a6f6e0d6781d1cb8689192408946d73-Abstract-Conference.html}
}

@inproceedings{mahloujifar2025auditing,
  author    = {Saeed Mahloujifar and
               Luca Melis and
               Kamalika Chaudhuri},
  title     = {Auditing {$f$}-Differential Privacy in One Run},
  booktitle = {Proceedings of the 42nd International Conference on Machine Learning},
  pages     = {42615--42641},
  year      = {2025},
  volume    = {267},
  series    = {Proceedings of Machine Learning Research},
  publisher = {PMLR},
  url       = {https://proceedings.mlr.press/v267/mahloujifar25a.html}
}

@inproceedings{nasr2023tight,
  author    = {Milad Nasr and
               Jamie Hayes and
               Thomas Steinke and
               Borja Balle and
               Florian Tram{\`e}r and
               Matthew Jagielski and
               Nicholas Carlini and
               Andreas Terzis},
  title     = {Tight Auditing of Differentially Private Machine Learning},
  booktitle = {32nd USENIX Security Symposium (USENIX Security 23)},
  pages     = {1631--1648},
  year      = {2023},
  publisher = {USENIX Association},
  url       = {https://www.usenix.org/conference/usenixsecurity23/presentation/nasr}
}

@inproceedings{xiang2025bits,
  author    = {Zihang Xiang and
               Tianhao Wang and
               Di Wang},
  title     = {Privacy Audit as Bits Transmission: {(Im)possibilities} for Audit by One Run},
  booktitle = {34th USENIX Security Symposium (USENIX Security 25)},
  year      = {2025},
  publisher = {USENIX Association},
  url       = {https://www.usenix.org/conference/usenixsecurity25/presentation/xiang-zihang}
}

@article{xiang2025tight,
  author        = {Zihang Xiang and
                   Tianhao Wang and
                   Hanshen Xiao and
                   Yuan Tian and
                   Di Wang},
  title         = {Tight Privacy Audit in One Run},
  journal       = {CoRR},
  volume        = {abs/2509.08704},
  year          = {2025},
  doi           = {10.48550/arXiv.2509.08704},
  eprint        = {2509.08704},
  archivePrefix = {arXiv},
  url           = {https://arxiv.org/abs/2509.08704}
}

@inproceedings{TPCBalleBG20,
  author    = {Borja Balle and
               Gilles Barthe and
               Marco Gaboardi},
  title     = {Privacy Amplification by Subsampling: Tight Analyses via Couplings and Divergences},
  booktitle = {Advances in Neural Information Processing Systems},
  volume    = {31},
  pages     = {6280--6290},
  year      = {2018}
}

@inproceedings{TCC:DMNS06,
  author    = {Cynthia Dwork and
               Frank McSherry and
               Kobbi Nissim and
               Adam Smith},
  title     = {Calibrating Noise to Sensitivity in Private Data Analysis},
  booktitle = {Theory of Cryptography},
  series    = {Lecture Notes in Computer Science},
  volume    = {3876},
  pages     = {265--284},
  year      = {2006},
  publisher = {Springer},
  doi       = {10.1007/11681878_14},
  url       = {https://doi.org/10.1007/11681878_14}
}

@article{chan2011private,
  author    = {T.{-}H. Hubert Chan and
               Elaine Shi and
               Dawn Song},
  title     = {Private and Continual Release of Statistics},
  journal   = {ACM Transactions on Information and System Security},
  volume    = {14},
  number    = {3},
  pages     = {1--24},
  year      = {2011},
  publisher = {ACM},
  doi       = {10.1145/2043621.2043626},
  url       = {https://doi.org/10.1145/2043621.2043626}
}

@misc{dpsynthetic,
  author        = {Alexey Kurakin and
                   Natalia Ponomareva and
                   Umar Syed and
                   Liam MacDermed and
                   Andreas Terzis},
  title         = {Harnessing Large-Language Models to Generate Private Synthetic Text},
  year          = {2024},
  eprint        = {2306.01684},
  archivePrefix = {arXiv},
  primaryClass  = {cs.LG},
  url           = {https://arxiv.org/abs/2306.01684}
}

@inproceedings{USENIX:ZWLHBH21,
  author    = {Zhikun Zhang and
               Tianhao Wang and
               Ninghui Li and
               Jean Honorio and
               Michael Backes and
               Shibo He and
               Jiming Chen and
               Yang Zhang},
  title     = {{PrivSyn}: Differentially Private Data Synthesis},
  booktitle = {30th USENIX Security Symposium (USENIX Security 21)},
  year      = {2021},
  publisher = {USENIX Association},
  url       = {https://www.usenix.org/conference/usenixsecurity21/presentation/zhang-zhikun}
}

@inproceedings{cvprdpsynth,
  author    = {Reihaneh Torkzadehmahani and
               Peter Kairouz and
               Benedict Paten},
  title     = {{DP-CGAN}: Differentially Private Synthetic Data and Label Generation},
  booktitle = {2019 IEEE/CVF Conference on Computer Vision and Pattern Recognition Workshops (CVPRW)},
  pages     = {98--104},
  year      = {2019},
  publisher = {IEEE Computer Society},
  doi       = {10.1109/CVPRW.2019.00018},
  url       = {https://doi.org/10.1109/CVPRW.2019.00018}
}

@inproceedings{yoon2018pategan,
  author    = {Jinsung Yoon and
               James Jordon and
               Mihaela van der Schaar},
  title     = {{PATE-GAN}: Generating Synthetic Data with Differential Privacy Guarantees},
  booktitle = {International Conference on Learning Representations},
  year      = {2019},
  url       = {https://openreview.net/forum?id=S1zk9iRqF7}
}

@inproceedings{pmlr-v97-mckenna19a,
  author    = {Ryan McKenna and
               Daniel Sheldon and
               Gerome Miklau},
  title     = {Graphical-Model Based Estimation and Inference for Differential Privacy},
  booktitle = {Proceedings of the 36th International Conference on Machine Learning},
  pages     = {4435--4444},
  year      = {2019},
  volume    = {97},
  series    = {Proceedings of Machine Learning Research},
  publisher = {PMLR},
  url       = {https://proceedings.mlr.press/v97/mckenna19a.html}
}

@article{andreux2020kymatio,
  title={Kymatio: Scattering transforms in python},
  author={Andreux, Mathieu and Angles, Tom{\'a}s and Exarchakis, Georgios and Leonarduzzi, Roberto and Rochette, Gaspar and Thiry, Louis and Zarka, John and Mallat, St{\'e}phane and And{\'e}n, Joakim and Belilovsky, Eugene and others},
  journal={Journal of Machine Learning Research},
  volume={21},
  number={60},
  pages={1--6},
  year={2020}
}

@book{casella2002statistical,

  title={Statistical Inference},

  author={Casella, George and Berger, Roger L.},

  year={2002},

  publisher={Duxbury}

}

@book{efron1994introduction,

  title={An Introduction to the Bootstrap},

  author={Efron, Bradley and Tibshirani, Robert J.},

  year={1994},

  publisher={Chapman and Hall/CRC}

}

@inproceedings{shokri2017membership,
  title={Membership inference attacks against machine learning models},
  author={Shokri, Reza and Stronati, Marco and Song, Congzheng and Shmatikov, Vitaly},
  booktitle={2017 IEEE symposium on security and privacy (SP)},
  pages={3--18},
  year={2017},
  organization={IEEE}
}

@inproceedings{carlini2019secret,
  title={The secret sharer: Evaluating and testing unintended memorization in neural networks},
  author={Carlini, Nicholas and Liu, Chang and Erlingsson, {\'U}lfar and Kos, Jernej and Song, Dawn},
  booktitle={28th USENIX security symposium (USENIX security 19)},
  pages={267--284},
  year={2019}
}

@inproceedings{yeom2018privacy,
  title={Privacy risk in machine learning: Analyzing the connection to overfitting},
  author={Yeom, Samuel and Giacomelli, Irene and Fredrikson, Matt and Jha, Somesh},
  booktitle={2018 IEEE 31st computer security foundations symposium (CSF)},
  pages={268--282},
  year={2018},
  organization={IEEE}
}

@book{petrov2012sums,
  title={Sums of independent random variables},
  author={Petrov, Valentin V},
  year={2012},
  publisher={Springer Science \& Business Media}
}

@incollection{petrov2000classical,
  title={Classical-type limit theorems for sums of independent random variables},
  author={Petrov, Valentin V},
  booktitle={Limit theorems of probability theory},
  pages={1--24},
  year={2000},
  publisher={Springer}
}

@inproceedings{zanella2023bayesian,
  author    = {Santiago Zanella-B{\'e}guelin and Lukas Wutschitz and Shruti Tople and Ahmed Salem and Victor R{\"u}hle and Andrew Paverd and Mohammad Naseri and Boris K{\"o}pf and Daniel Jones},
  title     = {Bayesian Estimation of Differential Privacy},
  booktitle = {Proceedings of the 40th International Conference on Machine Learning (ICML)},
  year      = {2023}
}

@inproceedings{annamalai2024nearly,
  author    = {Meenatchi Sundaram Muthu Selva Annamalai and Emiliano De Cristofaro},
  title     = {Nearly Tight Black-Box Auditing of Differentially Private Machine Learning},
  booktitle = {Advances in Neural Information Processing Systems (NeurIPS)},
  year      = {2024}
}

@article{liu2025enhancing,
  title={Enhancing one-run privacy auditing with quantile regression-based membership inference},
  author={Liu, Terrance and Boglioni, Matteo and Fu, Yiwei and Hu, Shengyuan and Thaker, Pratiksha and Wu, Zhiwei Steven},
  journal={arXiv preprint arXiv:2506.15349},
  year={2025}
}

@article{keinan2025well,
  title={How Well Can Differential Privacy Be Audited in One Run?},
  author={Keinan, Amit and Shenfeld, Moshe and Ligett, Katrina},
  journal={arXiv preprint arXiv:2503.07199},
  year={2025}
}

@inproceedings{kairouz2021practical,
  title={Practical and private (deep) learning without sampling or shuffling},
  author={Kairouz, Peter and McMahan, Brendan and Song, Shuang and Thakkar, Om and Thakurta, Abhradeep and Xu, Zheng},
  booktitle={International Conference on Machine Learning},
  pages={5213--5225},
  year={2021},
  organization={PMLR}
}

@inproceedings{andrew2024oneshot,
  author    = {Galen Andrew and Peter Kairouz and Sewoong Oh and Alina Oprea and H. Brendan McMahan and Vinith M. Suriyakumar},
  title     = {One-shot Empirical Privacy Estimation for Federated Learning},
  booktitle = {The Twelfth International Conference on Learning Representations (ICLR)},
  year      = {2024}
}

@inproceedings{maddock2023canife,
  author    = {Samuel Maddock and Alexandre Sablayrolles and Pierre Stock},
  title     = {{CANIFE}: Crafting Canaries for Empirical Privacy Measurement in Federated Learning},
  booktitle = {The Eleventh International Conference on Learning Representations (ICLR)},
  year      = {2023}
}

@article{tramer2022debugging,
  title={Debugging differential privacy: A case study for privacy auditing},
  author={Tramer, Florian and Terzis, Andreas and Steinke, Thomas and Song, Shuang and Jagielski, Matthew and Carlini, Nicholas},
  journal={arXiv preprint arXiv:2202.12219},
  year={2022}
}

@article{cebere2026privacy,
  title={Privacy in theory, bugs in practice: Grey-box auditing of differential privacy libraries},
  author={Cebere, Tudor and Erb, David and Desfontaines, Damien and Bellet, Aur{\'e}lien and Fitzsimons, Jack},
  journal={arXiv preprint arXiv:2602.17454},
  year={2026}
}

@InProceedings{pmlr-v235-chua24a,
  title = 	 {How Private are {DP}-{SGD} Implementations?},
  author =       {Chua, Lynn and Ghazi, Badih and Kamath, Pritish and Kumar, Ravi and Manurangsi, Pasin and Sinha, Amer and Zhang, Chiyuan},
  booktitle = 	 {Proceedings of the 41st International Conference on Machine Learning},
  pages = 	 {8904--8918},
  year = 	 {2024},
  editor = 	 {Salakhutdinov, Ruslan and Kolter, Zico and Heller, Katherine and Weller, Adrian and Oliver, Nuria and Scarlett, Jonathan and Berkenkamp, Felix},
  volume = 	 {235},
  series = 	 {Proceedings of Machine Learning Research},
  month = 	 {21--27 Jul},
  publisher =    {PMLR},
  url = 	 {https://proceedings.mlr.press/v235/chua24a.html},
}

@inproceedings{
adaptiveclipping,
title={Differentially Private Learning with Adaptive Clipping},
author={Galen Andrew and Om Thakkar and Hugh Brendan McMahan and Swaroop Ramaswamy},
booktitle={Advances in Neural Information Processing Systems},
editor={A. Beygelzimer and Y. Dauphin and P. Liang and J. Wortman Vaughan},
year={2021},
url={https://openreview.net/forum?id=RUQ1zwZR8_}
}

@inproceedings{ballewang,
  author    = {Borja Balle and
               Yu{-}Xiang Wang},
  title     = {Improving the {G}aussian Mechanism for Differential Privacy:
               Analytical Calibration and Optimal Denoising},
  booktitle = {Proceedings of the 35th International Conference on Machine Learning},
  pages     = {394--403},
  year      = {2018},
  volume    = {80},
  series    = {Proceedings of Machine Learning Research},
  publisher = {PMLR},
  url       = {https://proceedings.mlr.press/v80/balle18a.html}
}
